\documentclass{article}

\usepackage{url, hyperref}
\usepackage{amsthm, amsmath, amssymb} 
\usepackage{enumerate, paralist}
\usepackage{comment}
\usepackage{color}
\usepackage{multirow}
\usepackage{xspace}
\usepackage{rotating}
\usepackage{graphicx}

\newcommand{\para}[1]{\textbf{#1}~}

\newtheorem{theorem}{\textbf{Theorem}}
\newtheorem{lemma}[theorem]{\textbf{Lemma}}
\newtheorem{corollary}[theorem]{Corollary}
\newtheorem{proposition}[theorem]{Proposition}

\newtheorem{conjecture}[theorem]{Conjecture}

\newtheorem{definition}{Definition}
\newtheorem{assumption}{Assumption}

\renewcommand{\cite}{\citep}

\newcommand{\EE}{\mathbb{E}}

\newcommand{\VV}{\mathbb{V}}

\newcommand{\cond}{~\big\vert~}

\newcommand{\Indi}{\mathbb{I}}
\newcommand{\holder}{\,\cdot\,}

\newcommand{\maxnorm}[1]{\left\| #1 \right\|_\infty}

\DeclareMathOperator*{\argmin}{arg\,min}
\DeclareMathOperator*{\argmax}{arg\,max}

\newcommand{\Fcal}{\mathcal{F}}

\newcommand{\Gcal}{\mathcal{G}}
\newcommand{\Mcal}{\mathcal{M}}
\newcommand{\Scal}{\mathcal{S}}
\newcommand{\Sphi}{\mathcal{S}_\phi}
\newcommand{\Zcal}{\mathcal{Z}}
\newcommand{\Acal}{\mathcal{A}}
\newcommand{\Tcal}{\mathcal{T}}

\newcommand{\Rmax}{R_{\max}}
\newcommand{\Vmax}{V_{\max}}
\newcommand{\RR}{\mathbb{R}}
\newcommand{\Xcal}{\mathcal{X}}
\newcommand{\emp}[1]{\widehat{#1}}

\newcommand{\init}{\eta_1} 

\newcommand{\That}{\widehat{\mathcal{T}}_{\Fcal}}
\newcommand{\fh}{f_k}
\newcommand{\fo}{f_{k-1}}

\newcommand{\muu}{{\mu}}

\newcommand{\Lcal}{\mathcal{L}}

\newcommand{\smax}{s_{+}}
\newcommand{\smin}{s_{-}}

\newcount\Comments  
\definecolor{darkgreen}{rgb}{0,0.5,0}
\definecolor{darkred}{rgb}{0.7,0,0}
\definecolor{teal}{rgb}{0.3,0.8,0.8}
\newcommand{\kibitz}[2]{\ifnum\Comments=1\textcolor{#1}{#2}\fi}

\usepackage{microtype}
\usepackage{graphicx}
\usepackage{subfigure}
\usepackage{booktabs} 
\usepackage{enumitem}

\usepackage{hyperref}

\usepackage[accepted]{icml2019}

\icmltitlerunning{Information-Theoretic Considerations in Batch RL}

\begin{document}

\twocolumn[
\icmltitle{Information-Theoretic Considerations in Batch Reinforcement Learning}



\icmlsetsymbol{equal}{*}

\begin{icmlauthorlist}
\icmlauthor{Jinglin Chen}{uiuc}
\icmlauthor{Nan Jiang}{uiuc}
\end{icmlauthorlist}

\icmlaffiliation{uiuc}{University of Illinois at Urbana-Champaign, Urbana, Illinois, USA}

\icmlcorrespondingauthor{Nan Jiang}{nanjiang@illinois.edu}

\icmlkeywords{value-function approximation, information-theoretic lower bound}

\vskip 0.3in
]



\printAffiliationsAndNotice{}  

\begin{abstract}
Value-function approximation methods that operate in batch mode have foundational importance to reinforcement learning (RL). Finite sample guarantees for these methods often crucially rely on two types of assumptions: (1) mild distribution shift, and (2) representation conditions that are stronger than realizability. However, the necessity (``why do we need them?'') and the naturalness (``when do they hold?'') of such assumptions have largely eluded the literature. 
In this paper, we revisit these assumptions and provide theoretical results towards answering the above questions, 
and make steps towards a deeper understanding of value-function approximation. 
\end{abstract}

\section{Introduction and Related Work}
\label{sec:intro}

We are concerned with value-function approximation in batch-mode reinforcement learning, which is related to and sometimes known as Approximate Dynamic Programming \citep[ADP;][]{bertsekas1996neuro}. 
Such methods take sample transition data as input\footnote{In this paper, we restrict ourselves to the so-called \emph{one-path} setting and do not allow multiple samples from the same state \cite{sutton1998reinforcement, maillard2010finite}, which is only feasible in certain simulated environments and allows algorithms to succeed with realizability as the only representation condition.} and approximate the optimal value-function $Q^\star$ from a restricted class that encodes one's prior knowledge and inductive biases. They provide an important foundation for RL's empirical success today, as many popular deep RL algorithms find their prototypes in this literature. For example, when DQN \citep{mnih2015human} is run on off-policy data, and the target network is updated slowly, it can be viewed as the stochastic approximation of its batch analog, Fitted Q-Iteration \citep{ernst2005tree}, with a neural net as the function approximator \citep{riedmiller2005neural, yang2019theoretical}.

Given the importance of these methods, the question of \emph{when they work} is central to our understanding of RL. Existing works that analyze error propagation and finite sample behavior of ADP methods \citep{munos2003error, szepesvari2005finite, antos2008learning, munos2008finite, tosatto2017boosted} have provided us with a decent understanding: To guarantee sample-efficient learning of near-optimal policies, we often need assumptions from the following two categories.
 
\para{Mild distribution shift} Many ADP methods can run completely off-policy and they do the best with whatever data available.\footnote{Even when they are on-policy or combined with a standard exploration module (e.g., $\epsilon$-greedy), most often they fail in problems where exploration is difficult \citep[e.g., combination lock; see][]{kakade2003sample} and rely on the benignness of data to succeed.} Therefore, it is necessary that the data have sufficient coverage over the state (and action) space. 

\para{Representation condition} Since the ultimate goal is to find $Q^\star$, we would expect that the function class we work with contains it (or at least a close approximation). While such realizability-type assumptions are sufficient for supervised learning, reinforcement learning faces the additional difficulties of delayed consequences and the lack of labels, and existing analyses often make stronger assumptions on the function class, such as (approximate) closedness under Bellman update \cite{szepesvari2005finite}.

While the above assumptions make intuitive sense, and finite sample bounds have been proved when they hold, their necessity (``\emph{can we prove similar results without making these assumptions?}'') and naturalness (``\emph{do they actually hold in interesting problems?}'') have largely eluded the literature. 
In this paper, we revisit these assumptions and provide theoretical results towards answering the above questions. Below is a highlight of our results:
\begin{enumerate}[leftmargin=*]
\item To prepare for later discussions, we provide an analysis of representative ADP algorithms (FQI and its variant) under a simplified and minimal setup (Section~\ref{sec:upper}). As a side-product, our results improve upon prior analyses in the dependence of error rate on sample size.
\item We formally justify the necessity of mild distribution shift via an information-theoretic lower bound (Section~\ref{sec:con_lower}). Our setup rules out trivial and uninteresting failure mode due to an adversarial choice of data: Even with the most favorable data distribution, polynomial sample complexity is not achievable if the MDP dynamics are not restricted. 
\item We conjecture an information-theoretic lower bound against realizability alone as the representation condition (Conjecture~\ref{conj:complete}, Section~\ref{sec:lower_complete}). While we are not able to prove the conjecture, important steps are made, as two very general proof styles are shown to be destined to fail, one of which is due to \citet{sutton2018reinforcement} and has been used to prove a closely related result.
\item 
As another side-product, we prove that \emph{model-based RL} can enjoy polynomial sample complexity with realizability alone (Corollary~\ref{cor:mbrl}). If Conjecture~\ref{conj:complete} is true, we have a formal separation showing the gap between batch model-based vs value-based RL with function approximation (see the analog in the online exploration setting in \citet{sun2018model}).
\end{enumerate}
Throughout the paper, we make novel connections to two subareas of RL: state abstractions~\cite{whitt1978approximations, li2006towards} and PAC exploration under function approximation \cite{krishnamurthy2016pac, jiang2017contextual}. In particular, we are able to utilize some of their results in our proofs (Sections~\ref{sec:con_lower} and \ref{sec:lower_complete}), 
and find examples from these areas where the assumptions of interest hold (Sections~\ref{sec:visgrid} and \ref{sec:bisimulation}). This suggests that the results in these other areas may be beneficial to the research in ADP, and we hope this work can inspire researchers from different subareas of RL to exchange ideas more often.

\section{Preliminaries} \label{sec:prelim}

\subsection{Markov Decision Processes (MDPs)}
Let $M = (\Scal, \Acal, P, R, \gamma, \eta_1)$ be an MDP, where $\Scal$ is the finite (but can be arbitrarily large) state space, $\Acal$ is the finite action space, $P: \Scal\times \Acal\to \Delta(\Scal)$ is the transition function ($\Delta(\cdot)$ is the probability simplex), $R: \Scal\times\Acal \to [0, \Rmax]$ is the reward function, $\gamma \in [0, 1)$ is the discount factor, and $\eta_1$ is the initial distribution over states. 

A (stochastic) policy $\pi: \Scal\to\Delta(\Acal)$ prescribes a distribution over actions for each state. Fixing a start state $s$, the policy $\pi$ induces a random trajectory $s_1, a_1, r_1, s_2, a_2, r_2, \ldots$, where $s_1 = s$, $a_1 \sim \pi(s_1)$, $r_1 = R(s_1, a_1)$, $s_2 \sim P(s_1, a_1)$, $a_2 \sim \pi(s_2)$, etc. The goal is to find $\pi$ that maximizes the expected return $v^\pi := \EE[\sum_{h=1}^\infty \gamma^{h-1} r_h | s_1 \sim \init, \pi]$. It will also be useful to define the value function $V^\pi(s) := \EE[\sum_{h=1}^\infty \gamma^{h-1} r_h | s_1 = s, \pi]$ and Q-value function $Q^\pi(s,a) := \EE[\sum_{h=1}^\infty \gamma^{h-1} r_h | s_1 = s, a_1 = a, a_{2:\infty} \sim \pi]$, and these functions take values in $[0, \Vmax]$ with $\Vmax:=\Rmax/(1-\gamma)$.

There exists a deterministic policy\footnote{A deterministic policy puts all the probability mass on a single action in each state. With a slight abuse of notation, we sometimes also treat the type of such policies as $\pi: \Scal\to\Acal$.} $\pi^\star$ that maximizes $V^{\pi}(s)$ for all $s\in\Scal$ simultaneously, and hence also maximizes $v^\pi$ as $v^\pi = \EE_{s_1 \sim \init}[V^\pi(s_1)]$. Let $V^\star$ and $Q^\star$ be the shorthand for $V^{\pi^\star}$ and $Q^{\pi^\star}$ respectively. It is well known that $\pi^\star(s) = \pi_{Q^\star}(s) := \argmax_{a\in\Acal} Q^\star(s,a)$, and $Q^\star$ satisfies the \emph{Bellman equation} $Q^\star = \Tcal Q^\star$, where $\Tcal: \RR^{\Scal\times\Acal}\to \RR^{\Scal\times \Acal}$ is the \emph{Bellman update operator}: $\forall f \in \RR^{\Scal\times\Acal}$,
\begin{align} \label{eq:bellman}
	(\Tcal f)(s,a) := R(s,a) + \gamma \EE_{s'\sim P(s,a)} [V_f(s')],
\end{align}
where $V_f(s') := \max_{a' \in \Acal} f(s', a')$.

\para{Additional notations} Let $\eta^{\pi}_h$ be the marginal distribution of $s_h$ under $\pi$, that is, $\eta^{\pi}_h(s) := \Pr[s_h = s \cond s_1 \sim \eta_1, \pi]$. For $g: \Scal\times\Acal \to \RR$, $\nu \in \Delta(\Scal\times\Acal)$, and $p\ge 1$, define the shorthand $\|g\|_{p, \nu} := (\EE_{(s,a) \sim \nu}[|g(s,a)|^p])^{1/p}$, which is a semi-norm. Furthermore, for any object that is a function of/distribution over $\Scal$ (or $\Scal\times\Acal$), we will treat it as a vector whenever convenient. We add a subscript to the value functions or Bellman update operators, e.g., $V_M^\star$, when it is necessary to clarify the MDP in which the object is defined.

\subsection{Batch Value-Function Approximation}
This paper is concerned with \emph{batch-mode} RL with value-function approximation. As a typical setup, the agent does not have direct access to the MDP and instead is given the following inputs:
\begin{itemize}[leftmargin=*]
	\item A batch dataset $D$ consisting of $(s,a,r,s')$ tuples, where $r = R(s,a)$ and $s' \sim P(s,a)$. For simplicity, we assume that $(s,a)$ is generated i.i.d.~from the \emph{data distribution} $\mu\in\Delta(\Scal \times \Acal)$.\footnote{The agent may or may not have knowledge of $\mu$. Most existing algorithms are agnostic to such knowledge.}
	\item A class of candidate value-functions, $\Fcal \subset (\Scal\times\Acal\to [0, \Vmax])$, which (approximately) captures $Q^\star$; such a property is often called \emph{realizability}. We discuss additional assumptions on $\Fcal$ later. As a further simplification, we focus on finite but exponentially large $\Fcal$ and discuss how to handle infinite classes when appropriate.
\end{itemize}
The learning goal is to compute a near-optimal policy from the data, often via finding $f\in\Fcal$ that approximates $Q^\star$ and outputting $\pi_f$, the greedy policy w.r.t.~$f$. 
A representative algorithm for this setting is Fitted Q-Iteration (FQI) \citep{ernst2005tree, szepesvari2010algorithms}.\footnote{Batch value-based algorithms can often be categorized into approximate value iteration (e.g., FQI) and approximate policy iteration (e.g., LSPI \citep{lagoudakis2003least}). We focus on the former due to its simplicity and do not discuss the latter as its guarantees often rely on similar but more complicated assumptions \citep{lazaric2012finite}. Moreover, our lower bounds are information-theoretic and algorithm-independent.} The algorithm initializes $f_0 \in \Fcal$ arbitrarily, and iteratively computes $\fh$ as follows: in iteration $k$, the algorithm converts the dataset $D$ into a regression dataset, with $(s,a)$ being the input and $r+ \gamma V_{\fo}(s')$ as the output. It then minimizes the squared loss regression objective over $\Fcal$, and the minimizer becomes $\fh$. More formally, $\fh:= \That \fo$, where
\begin{align} \label{eq:Lcal_D} 
	\That f' := &~\argmin_{f \in\Fcal} \Lcal_D(f; f')\\
\Lcal_D(f; f') := &~\frac{1}{|D|} \sum_{(s,a,r,s') \in D} \left(f(s,a) - r - \gamma V_{f'}(s')\right)^2. \nonumber 
\end{align}
FQI may oscillate and a fixed point solution may not exist in general \citep{gordon1995stable}. Nevertheless, under conditions which we will specify later, finite sample guarantees for FQI can still be  obtained even if the process does not converge.

\subsection{State Abstractions} \label{sec:abstraction}
A state abstraction $\phi$ maps $\Scal$ to a finite and potentially much smaller abstract state space, $\Sphi$. Naturally, $\phi$ is often a many-to-one mapping, inducing an equivalence notion over $\Scal$ which encodes one's prior knowledge of equivalent or similar states. A typical use of abstractions in the batch learning setting is to construct a tabular (or \emph{certainty-equivalent}) model from a dataset $\{(\phi(s), a, r, \phi(s'))\}$, and compute the optimal policy in the resulting abstract model. There is a long history of studying abstractions, mostly focusing on their approximation guarantees \cite{whitt1978approximations}. 

We note, however, that there is a direct connection between FQI and certainty-equivalence with abstractions. In particular, value iteration in the model estimated with abstraction $\phi$ is \emph{exactly equivalent} to FQI with $\Fcal$ being the class of piece-wise constant functions under $\phi$.\footnote{This result is known anecdotally \citep[see e.g.,][]{pires2016policy} and we include details in Appendix~\ref{app:abstraction} for completeness.} As such, the characterization of approximation errors in the two bodies of literature are closely related to each other. We will discuss further connections in the rest of this paper.

\section{Bellman Error Minimization in Batch Reinforcement Learning}
\label{sec:upper}
In this section, we give a complete analysis of FQI and a related algorithm, with the main results being two sample complexity bounds. Many of the insights and results in this section have either explicitly appeared in or been implicitly hinted by prior work \citep[especially][]{szepesvari2005finite, antos2008learning}, and we include them because (1) the discussions in the rest of the paper are largely based on these results, and (2) our analyses simplify prior results without trivializing them, making the high-level insights more accessible. We also improve the results in some aspects.

\subsection{Sample-Based Bellman Error Minimization} \label{sec:minimax}
We start by deriving FQI from a slightly unusual perspective due to the aforementioned prior work, which motivates major assumptions in FQI analysis and introduces concepts that are important for later discussions. 

Recall that the goal of value-based RL is to find $f \in\Fcal$ such that $f \approx \Tcal f$, that is, $\|f - \Tcal f\|=0$ where $\|\cdot\|$ is some appropriate norm. For example, if $\mu$ is a distribution supported on the entire $\Scal\times \Acal$, then $\|f - \Tcal f\|_{2,\mu}^2 = 0$ 
would guarantee that $f = Q^\star$. While such an $f$ can be found in principle by minimizing
$\|f -\Tcal f\|_{2, \mu}^2$ 
over $f\in\Fcal$, calculating $\|f - \Tcal f\|$ requires knowledge of the transition dynamics (recall Eq.\eqref{eq:bellman}), which is unknown in the learning setting. Instead, we have access to the dataset $D = \{(s,a,r,s')\}$, and it may be tempting to minimize the following objective that is purely a function of data: (Recall $\Lcal_D$ in Eq.\eqref{eq:Lcal_D})
\begin{align*}
	\Lcal_D(f; f) := &~\frac{1}{|D|} \sum_{(s,a,r,s') \in D} \left(f(s,a) - r - \gamma V_{f}(s')\right)^2.
\end{align*}
Unfortunately, even with the infinite amount of data, the above objective is still different from the actual Bellman error $\|f -\Tcal f\|_{2, \mu}^2$ that we wish to minimize. In particular, define
$
\Lcal_\mu(\cdot; \cdot) := \EE[\Lcal_D(\cdot; \cdot)],
$ 
where the expectation is w.r.t.~the random draw of the dataset $D$. 
We have $	\Lcal_\mu(f; f) = $
\begin{align} \label{eq:cond_var}
\|f - \Tcal f\|_{2, \mu}^2 + \gamma^2 \EE_{(s,a) \sim \mu}[\VV_{s' \sim P(s,a)}[V_f(s')]].
\end{align}
In words, $\Lcal_\mu(f; f)$ adds a conditional variance term to the desired objective, which incorrectly penalizes functions that have a large variance w.r.t.~random state transitions. 

\para{The minimax algorithm}\footnote{Also known under the name ``modified Bellman Residual Minimization'' \cite{antos2008learning}.}
One way to fix the issue is to estimate the conditional variance term in Eq.~\eqref{eq:cond_var} and subtracting it from $\Lcal_D(f;f)$. In fact, it is easy to verify that $\gamma^2 \EE_{(s,a) \sim \mu}[\VV_{s' \sim P(s,a)}[V_f(s')]]$ is the Bayes optimal error of the regression problem
\begin{align} \label{eq:regress}
	(s,a) \mapsto r + \gamma V_{f}(s').
\end{align}
One can estimate it by empirical risk minimization over a rich function class, and the estimate is consistent as long as the function class realizes the Bayes optimal regressor and has bounded statistical complexity. 
Following this idea, we assume access to another function class $\Gcal \subset (\Scal \times \Acal \to [0, \Vmax])$ for solving the regression problem in Eq.\eqref{eq:regress}. The estimated Bayes optimal error is
\begin{align} \label{eq:bayes_err_estm}
	\inf_{g\in\Gcal} \Lcal_D(g; f).
\end{align}
A good approximation to $\|f - \Tcal f\|_{2,\mu}^2$ from data is then
$
\sup_{g\in\Gcal} \Lcal_D(f; f) - \Lcal_D(g; f).
$ 
This suggests that we can simply run the following optimization problem to find $f\in \Fcal$ that approximates $Q^\star$:
\begin{align} \label{eq:minimax}
	\inf_{f\in\Fcal} \sup_{g\in\Gcal} \Lcal_D(f; f) - \Lcal_D(g; f).
\end{align}
Later in this section, we will provide a finite sample analysis of the above minimax algorithm, but before that, we will show that FQI can be viewed as its approximation. 

\para{FQI as an approximation to Eq.\eqref{eq:minimax}} FQI has a close connection to the above program and can be viewed as its approximation, when $\Gcal$ is chosen to be $\Fcal$. Formally, 
\begin{proposition} \label{prop:fqi_minimax}
	Let $\hat f$, $\hat g$ be the solution to Eq.\eqref{eq:minimax} when $\Gcal=\Fcal$. 
	\begin{itemize}[leftmargin=*]
		\item If $\Lcal_D(\hat f; \hat f) - \Lcal_D(\hat g; \hat f)= 0$, $\hat f$ is a fixed point for FQI. 
		\item Conversely, if $\fh=\fo$ holds for some $k$ in FQI, then $\hat f = \hat g = \fh$ is a solution to Eq.\eqref{eq:minimax}.
		\item If $\Lcal_D(\hat f; \hat f) - \Lcal_D(\hat g; \hat f)> 0$, FQI oscillates and no fixed point exists. 
	\end{itemize}
\end{proposition}
The proof is deferred to Appendix~\ref{app:minimax_opt}.
The proposition states that the minimax algorithm is more stable than FQI, and when FQI reaches a fixed point, the solutions of the two algorithms coincide. In fact, \citet{dai2018sbeed} derives a closely related algorithm using Fenchel dual and shows that the algorithm is always convergent.



\subsection{Analysis of FQI and Its Minimax Variant}
We provide finite sample guarantees to the two algorithms introduced above; closely related analyses have appeared in prior works (see Section~\ref{sec:intro} for references), and our version provides a cleaner analysis under simplification assumptions, improves the error rate as a function of sample size, and prepares us for later discussions. 

To state the guarantees, we need to introduce the two assumptions that are core to this paper. The first assumption handles distribution shift, and we precede it with the definition of \emph{admissible distributions}.

\begin{definition}[Admissible distributions]
	We say a distribution $\nu \in \Delta(\Scal\times\Acal)$ is \emph{admissible} in MDP $M = (\Scal, \Acal, P, R, \gamma, \eta_1)$, if there exists $h\ge 0$, and a (potentially non-stationary and stochastic) policy $\pi$, such that $\nu(s,a) = \Pr[s_h = s, a_h = a | s_1 \sim \init, \pi]$. 
\end{definition}

Intuitively, a distribution is admissible if it can be generated in the MDP by following some policy for a number of timesteps. The following assumption on \emph{concentratability} asserts that all admissible distributions are not ``far away'' from the data distribution $\mu$. The original definition is due to  \citet{munos2003error}.

\begin{assumption}[Concentratability coefficient] \label{asm:concentratability}
	We assume that there exists $C < \infty$ s.t.~for any admissible $\nu$, 
	$$
	\forall (s,a)\in\Scal\times\Acal,~ \frac{\nu(s,a)}{\mu(s,a)} \le C.
	$$ 
\end{assumption}

The real (and implicit) assumption here is that $C$ is manageably large, as our sample complexity bounds scale linearly with $C$. Prior works have used more sophisticated definitions \citep{farahmand2010error}.\footnote{This often comes at the cost of their bound being not \emph{a priori}, i.e., having a dependence on the randomness of data, initialization, and tie-breaking in optimization.} The technicalities introduced are largely orthogonal to the discussions in this paper, so we choose to adopt a much  simplified version. Despite the simplification, we will see natural examples that yield small $C$ under our definition in Section~\ref{sec:concentratability}. We will also discuss how to relax it using the structure of $\Fcal$ at the end of the paper.

Next, we introduce the assumption on the representation power of $\Fcal$ and $\Gcal$. 

\begin{assumption}[Realizability] \label{asm:realizability}
	$Q^\star \in \Fcal$. \\ (When this holds approximately, we measure violation by $\epsilon_{\Fcal} := \inf_{f\in\Fcal}\|f-\Tcal f\|_{2,\mu}^2$.) 
\end{assumption}

\begin{assumption}[Completeness] \label{asm:completeness}
	$\forall f \in \Fcal$, $\Tcal f \in \Gcal$. \\(When this holds approximately, we measure  violation by $\epsilon_{\Fcal, \Gcal} := \sup_{f\in\Fcal} \inf_{g \in \Gcal} \|g - \Tcal f\|_{2, \mu}^2$.) 
\end{assumption}
These assumptions lead to finite sample guarantees for both the minimax algorithm and FQI. For FQI, since $\Gcal=\Fcal$, Assumption~\ref{asm:completeness} essentially states that \emph{$\Fcal$ is closed under operator $\Tcal$}, hence ``completeness''.\footnote{In the literature, the violation of completeness when $\Fcal=\Gcal$, $\epsilon_{\Fcal,\Fcal}$, is called \emph{inherent Bellman error}.} The assumption is natural from how we derive the minimax algorithm in Sec~\ref{sec:minimax}, as Eq.\eqref{eq:bayes_err_estm} is only a consistent estimate of the Bayes optimal error of Eq.\eqref{eq:regress} if $\Gcal$ realizes the Bayes optimal regressor, which is $\Tcal f$.

A few remarks in order: 
\begin{enumerate}[leftmargin=*]
	\item When $\Fcal=\Gcal$ is finite, completeness implies realizability.\footnote{This is because $\Tcal^k f$ never repeats itself, as its $\ell_\infty$ distance to $Q^\star$ shrinks exponentially with a rate of $\gamma$ due to contraction.} However, completeness is stronger and much less desired than realizability: realizability is monotone in $\Fcal$ (adding functions to $\Fcal$ never hurts realizability), while completeness is not (adding functions to $\Fcal$ may break completeness). 
	\item While we focus on completeness, it is not the only condition that leads to guarantees for  ADP algorithms. We discuss alternative assumptions in Section~\ref{sec:conclusion}.
\end{enumerate}

Now we are ready to state the sample complexity results. In Appendices~\ref{app:fqi} and \ref{app:minimax} we provide more general error bounds (Theorems~\ref{thm:fqi_full} and \ref{thm:minimax_full}) that handle the approximate case where $\epsilon_\Fcal$ and $\epsilon_{\Fcal, \Gcal}$ are not zero and iteration $k$ is finite. To keep the main text focused and accessible, we only present their sample complexity corollaries in the exact case.
\textbf{}\begin{theorem}[Sample complexity of FQI] \label{thm:fqi}
	Given a dataset $D=\{(s,a,r,s')\}$ with sample size $|D|=n$ and $\Fcal$ that satisfies completeness (Assumption~\ref{asm:completeness} when $\Gcal=\Fcal$), w.p.~$\ge 1-\delta$, the output policy of FQI after $k$ iterations, $\pi_{f_k}$, satisfies
	$v^\star - v^{\pi_{\fh}} \le \epsilon \cdot \Vmax$ when $k\to \infty$ and\footnote{Only absolute constants are suppressed in Big-Oh notations.} 
	$\displaystyle
	n = O\left(\frac{C\ln\frac{|\Fcal|}{\delta}}{\epsilon^2(1-\gamma)^4}\right).
	$
\end{theorem}

\begin{theorem}[Sample complexity of the minimax variant] \label{thm:minimax}
	Given a dataset $D=\{(s,a,r,s')\}$ with sample size $|D|=n$ and $\Fcal$, $\Gcal$ that satisfy realizability (Assumption~\ref{asm:realizability}) and completeness (Assumption~\ref{asm:completeness}) respectively,  w.p.~$\ge 1-\delta$, the output policy of the minimax algorithm (Eq.\eqref{eq:minimax}), $\pi_{\hat f}$, satisfies $v^\star - v^{\pi_{\hat f}} \le \epsilon \cdot  \Vmax$, if
	$\displaystyle
	n = O\left(\frac{C\ln\frac{|\Fcal||\Gcal|}{\delta}}{\epsilon^2(1-\gamma)^4}\right).
	$
\end{theorem}

Our results show that the suboptimality $\epsilon$ decreases in the rate of $n^{-1/2}$ when realizability and completeness hold exactly, and the more general error bounds (Theorems~\ref{thm:fqi_full} and \ref{thm:minimax_full}) degrade gracefully from the exact case as $\epsilon_{\Fcal, \Fcal}$ (or $\epsilon_{\Fcal}$ and $\epsilon_{\Fcal, \Gcal}$) increases. This is obtained via the use of Bernstein's inequality to achieve fast rate in least square regression. 
While results similar to Theorems~\ref{thm:fqi} and \ref{thm:fqi_full} exist (\citealt[Chapter 5]{farahmand2011regularization}; see also \citet{lazaric2012finite, pires2012statistical, farahmand2016regularized}), according to our knowledge, fast rate for the minimax algorithm has not been established before: for example, \citet{antos2008learning, munos2008finite} obtain an error rate of $n^{-1/4}$ in closely related settings, but their rates do not improve to $n^{-1/2}$ in the absence of approximation.\footnote{Note however that they handle infinite function classes. 
In fact, \citet[pg.831]{munos2008finite} have discussed the possibility of an $n^{-1/2}$ result, which we obtain here. See the beginning of Appendix~\ref{app:fqi} for further discussions.} The major limitation of our result is the assumption of finite $\Fcal$ and $\Gcal$ due to our minimal setup, and we refer readers to \citet{yang2019theoretical} for a recent analysis that specializes in ReLU networks.\footnote{Their analysis modifies the FQI algorithm and samples fresh data in each iteration, dodging some of the technical difficulties due to reusing the same batch of data, which we handle here.}

We do not discuss the proofs in further details since the improvement in error rate is a side-product and this section is mainly meant to simplify prior analyses and provide a basis for subsequent discussions. Interested readers are invited to consult Appendices~\ref{app:fqi} and \ref{app:minimax} where we provide sketched outlines as well as detailed proofs. 

\section{On Concentratability}
\label{sec:concentratability}

In this section, we establish the necessity of Assumption~\ref{asm:concentratability} and show natural examples where concentratability is low.  
While it is easy to construct a counterexample of missing data\footnote{ That is, $\mu$ puts $0$ probability on important states and actions.} against removing Assumption~\ref{asm:concentratability}, such a counterexample only reflects a trivial failure mode due to an adversarial choice of data. What we show is a deeper and nontrivial failure mode: Even with the \emph{most favorable} data distribution, polynomial sample complexity is precluded if we put no restriction on MDP dynamics. This result improves our understanding on concentratability, and shows that this assumption is not only about the data distribution, but also (and perhaps more) about the environment and the state distributions induced therein. 

\subsection{Lower Bound} \label{sec:con_lower}
To show that low concentratability is necessary, we prove a hardness result, where both realizability and completeness hold, and an algorithm has the freedom to choose \emph{any} data distribution $\mu$ that is favorable, yet \emph{no algorithm} can achieve $poly(|\Acal|, \tfrac{1}{1-\gamma}, \ln|\Fcal|, \ln|\Gcal|, \tfrac{1}{\epsilon}, \tfrac{1}{\delta})$ sample complexity. Crucially, the  concentratability coefficient of any data distribution on the worst-case MDP is always exponential in horizon, so the lower bound does not conflict with the upper bounds in Section~\ref{sec:upper}, as the exponential sample complexity would have been explained away by the dependence on $C$.

\begin{theorem}
	\label{thm:lower_con}
	There exists a family of MDPs $\Mcal$ (they share the same $\Scal$, $\Acal$, $\gamma$), $\Fcal$ that realizes the $Q^\star$ of every MDP in the family, and $\Gcal$ that realizes $\Tcal_{M'} f$ for any $M'\in\Mcal$ and any $f\in\Fcal$, such that: for any data distribution and any batch algorithm with $(\Fcal, \Gcal)$ as input, an adversary can choose an MDP from the family, such that the sample complexity for the algorithm to find an $\epsilon$-optimal policy cannot be poly($|\Acal|,  \tfrac{1}{1-\gamma}, \ln|\Fcal|, \ln|\Gcal|, \tfrac{1}{\epsilon}, \tfrac{1}{\delta}$).
\end{theorem}

\begin{proof}
	We construct $\Mcal$, a family of hard MDPs, and prove the theorem via the combination of two arguments: 
	\begin{enumerate}[leftmargin=*]
		\item All algorithms are subject to an exponential lower bound (w.r.t.~the horizon) even if  (a) they have compact $\Fcal$ and $\Gcal$ that satisfy realizability and completeness as inputs, and (b) they can perform \emph{exploration} during data collection.
		\item Since the MDPs in the construction share the same deterministic transition dynamics, the combination of any data distribution and any batch RL algorithm is \emph{a special case} of an exploration algorithm.
	\end{enumerate}
	
	We first provide argument (1), which reuses the construction by \citet{krishnamurthy2016pac}. Let each instance of $\Mcal$ be a complete tree with branching factor $|\Acal|$ and depth $H = \lfloor 1/(1-\gamma) \rfloor$.  Transitions are deterministic, and only leaf nodes have non-zero rewards. All leaves give Ber$(1/2)$ rewards, except for one that gives Ber($1/2+\epsilon$). Changing the position of this optimal leaf yields a family of $|\Acal|^H$ MDPs, and in order to  achieve a suboptimality that is a constant fraction of $\epsilon$, the algorithm is required to identify this optimal leaf.\footnote{All leaf rewards are discounted by only a constant when $\gamma \to 1$, as $\gamma^{1/(1-\gamma)} \to e^{-1}$.} In fact, the problem is equivalent to the hard instances of best arm identification with $|\Acal|^H$ arms, so even if an algorithm can perform active exploration, the sample complexity is still $\Omega(|\Acal|^H \ln(1/\delta) / \epsilon^2)$ (see \citet{krishnamurthy2016pac} for details, who use standard techniques from \citet{auer2002finite}).
	
	Now we provide $\Fcal$ and $\Gcal$ that (1) satisfy Assumptions~\ref{asm:realizability} and \ref{asm:completeness}, (2) do not provide any information other than the fact that the problem is in $\Mcal$, and (3) have ``small'' logarithmic sizes so that $\ln|\Fcal|$ and $\ln|\Gcal|$ cannot explain away the exponential sample complexity. Let $\Fcal = \{Q_{M'}^\star: M' \in \Mcal\}$, where the subscript specifies the MDP with respect to which we compute $Q^\star$. Let $\Gcal = \{\Tcal_{M'} Q_{M''}^\star: M', M'' \in \Mcal \}$. Such $\Fcal$ and $\Gcal$ satisfy realizability and completeness by definition, and have statistical complexities $\ln|\Fcal|=H\ln|\Acal|$ and $\ln|\Gcal|\le 2H\ln|\Acal|$, respectively. With this, we conclude that any \emph{exploration} algorithm cannot obtain $poly(|\Acal|, \tfrac{1}{1-\gamma}, \ln|\Fcal|, \ln|\Gcal|, \tfrac{1}{\epsilon})$ sample complexity.
	
	We complete the proof with the second argument. Note that all the MDPs in $\Mcal$ only differ in leaf rewards and share the same deterministic transition dynamics. Therefore, a learner with the ability to actively explore can mimic the combination of \emph{any data distribution $\mu\in \Delta(\Scal\times\Acal)$ and any batch RL algorithm}, by (1) collecting data from $\mu$ (which is always doable due to known and deterministic transitions), and (2) running the batch algorithm after data is collected. This completes the proof.
\end{proof}

\subsection{Natural Examples} \label{sec:visgrid}
We have shown that polynomial learning is precluded if no restriction is put on the MDP dynamics, even if data is chosen in a favorable manner. The next question is, is low concentratability common, or at least found in interesting problems? In general, even if the data distribution $\mu$ is uniform over the state-action space, the worst-case $C$ might still scale with $|\Scal\times\Acal|$, which can be too large in challenging RL problems for the guarantees to be any meaningful. To this end, \citet{munos2007performance} has provided several carefully constructed tabular examples, demonstrating that $C$ does not always scale badly. However, are there more general problem families that capture RL scenarios found in empirical work, yet always yield a bounded $C$? 

\para{Example in problems with rich observations} We find answers to the above problem in recent development of PAC exploration in rich-observation problems \citep{krishnamurthy2016pac, jiang2017contextual, dann2018oracle}, where a general low-rank condition (a.k.a.~\emph{Bellman rank} \citep{jiang2017contextual}) has been identified that enables sample-efficient exploration under function approximation. One of the prominent examples where such a condition holds is inspired by ``visual gridworld'' environments in empirical RL research \citep[see e.g.,][]{johnson2016malmo}: the dynamics are defined over a small number of hidden states (e.g., grids), and the agent receives high dimensional observations that are generated i.i.d.~from the hidden states (e.g., raw-pixel images as observations). Below we show that in these environments, there always exists a data distribution that yields small $C$ for batch learning, and such a distribution can be naturally generated as a mixture of admissible distributions. We include an informal statement below, deferring the precise version and the proof to Appendix~\ref{app:visgrid}. 

\begin{proposition}[Informal] \label{prop:visgrid}
Let $M$ be a reactive POMDP as defined in \citet{jiang2017contextual}, where the underlying hidden state space $\Zcal$ is finite but the (Markov) observation space $\Scal$ can be arbitrarily large. There always exists a state-action distribution $\mu$ such that $C = |\Zcal\times\Acal|$ satisfies Assumption~\ref{asm:concentratability}. Furthermore, $\mu$ can be obtained by taking a probability mixture of several admissible distributions.
\end{proposition}

Similar results can be established for other structures studied by \citet{jiang2017contextual} (e.g., large MDPs with low-rank transitions), which we omit here. These results suggest that \emph{Bellman rank is the counterpart for concentratability coefficient in the online exploration setting}. Further implications and how to leverage this connection  to improve the definition of concentratability will be discussed in Section~\ref{sec:conclusion}.

\section{On Completeness}
\subsection{Towards an Information-Theoretic Lower Bound in the Absence of Completeness} \label{sec:lower_complete}
We would also like to establish the necessity of completeness by showing that, there exist hard MDPs that cannot be efficiently learned with value-function approximation, even under low concentratability and realizability (Assumptions~\ref{asm:concentratability} and \ref{asm:realizability}).\footnote{Note that the existence of such a lower bound would not imply that completeness is indispensable. Rather it simply states that realizability alone is insufficient, and we need stronger conditions on $\Fcal$, for which completeness is a candidate. } In fact, \emph{algorithm-specific} hardness results have been known for a long time \citep[see e.g.,][]{van1994feature, gordon1995stable, tsitsiklis1997analysis}, where ADP algorithms are shown to diverge even in MDPs with a small number of states, when the algorithm is forced to work with a restricted class of functions.\footnote{Interested readers can consult \citet{shipranotes}. See also \citet[Theorem 45]{dann2018oracle} for a more plain example.} Unfortunately, such hardness results are insufficient to confirm the fundamental difficulty of the problem, and it is important to seek \emph{information-theoretic} lower bounds. 

While we are not able to obtain such a lower bound, what we find is that the counterexample (if it exists) must be highly nontrivial and probably need ideas that are not present in standard statistical learning theory (SLT) and RL literature. More concretely, we show that two general proof styles are destined to fail in such a task, as polynomial sample complexity can be achieved information-theoretically.

\para{Exponential-sized model family will not work}
Standard lower bounds in SLT often start with the construction of a family of problem instances that has an exponential size \cite{yu1997assouad}.\footnote{In fact, our Theorem~\ref{thm:lower_con} also follows this style, whose construction is due to \citet{krishnamurthy2016pac, jiang2017contextual}.} We show that this will simply never work, which is a direct corollary of Theorem~\ref{thm:minimax}:

\begin{corollary}[Batch model-based RL only needs realizability] \label{cor:mbrl}
	Let $D=\{(s,a,r,s')\}$ be a dataset with sample size $|D|=n$, $C$ as defined in Assumption~\ref{asm:concentratability}, and $\Mcal$ a model class that realizes the true MDP $M$, i.e., $M\in\Mcal$. There exists  an (information-theoretic) algorithm that takes $\Mcal$ as input and return an $(\epsilon \Vmax)$-optimal policy w.p.~$\ge 1-\delta$, if
	$\displaystyle
	n = O\left(\frac{C\ln\frac{|\Mcal|}{\delta}}{\epsilon^2(1-\gamma)^4}\right).
	$
\end{corollary}
\begin{proof}
	We use the same idea as the proof of Theorem~\ref{thm:lower_con}: Let $\Fcal = \{Q_{M'}^\star: M'\in\Mcal\}$, and $\Gcal = \{\Tcal_{M'} Q_{M''}^\star: M', M''\in\Mcal\}$. Note that $\ln|\Fcal| \le \ln|\Mcal|$, and $\ln |\Gcal| \le 2\ln |\Mcal|$. $(\Fcal, \Gcal)$ satisfy both realizability and completeness, so we apply the minimax algorithm (Eq.\eqref{eq:minimax}) and the guarantee in Theorem~\ref{thm:minimax} immediately holds.
\end{proof}

Essentially, this result shows that batch model-based RL can succeed with realizability as the only representation condition for the model class, because we can reduce it to value-based learning and obtain completeness \emph{for free}. This illustrates a significant barrier to an algorithm-independent lower bound, that in an information-theoretic setting, the learner can always specialize in the family of hard instances and have the freedom to choose its algorithm style, thus can be \emph{model-based}. However, in the context of value-function approximation, it is obvious that we are assuming no prior knowledge of the model class and hence cannot run any model-based algorithm. How can we encode such a constraint mathematically?

\para{Tabular MDPs with a restricted value-function class will not work}
\citet[Section 11.6]{sutton2018reinforcement} proposes a clever way to prevent the learner to be model-based for linear function approximation, and a closely related definition is recently given by \citet{sun2018model} that applies to arbitrary function classes. 

The idea is the following: Instead of providing the dataset $D=\{(s,a,r,s')\}$ directly, we preprocess the data and mask the identity of $s$ (and $s'$). While $s$ is not directly observable, the learner can query the evaluation of any $f\in\Fcal$ on $s$ for any $a\in\Acal$. That is, we represent each state $s$ by its \emph{value profile}, $\{f(s,a): f\in\Fcal, a\in\Acal\}$. This definition agrees with intuition and can be used to express a wide range of popular algorithms, including FQI.

Using this definition, \citet{sutton2018reinforcement} proves a result closely related to what we aim at here: they show that the Bellman error $\|f - \Tcal f\|$ is not learnable. In particular, there exist two MDPs (with finite and constant-sized state space) and a value function, such that (1) a value-based learner (who only has access to the value profiles of states) cannot distinguish between the data coming from the two MDPs, and (2) the Bellman error of the value function is different in the two MDPs. 

While encouraging and promising, their constructions have a crucial caveat for our purpose, that the value function class is not realizable.\footnote{They force two states who have different optimal values to share the same features for linear function approximation.} With further investigation, we sadly find that such a caveat is fundamental: no information-theoretic lower bound can be shown if realizability holds in na\"ive tabular constructions with a constant-sized state-action space and uniform data, hence value profile cannot be the \emph{only} mechanism to induce hardness. In fact, we can prove a stronger result than we need here for $\Scal$ and $\Acal$ that are not necessarily constant-sized:
\begin{proposition} \label{prop:value_profile}
	Let $M$ be an MDP with a finite state space and $\Fcal$ a realizable function class. Given a dataset $D=\{(s,a,r,s')\}$ where each $(s,a)$ receives $\Omega(|D|/|\Scal\times\Acal|)$ samples, there exists an algorithm that only operates on states via their value profiles yet enjoy poly$(|\Scal|, |\Acal|, \tfrac{1}{1-\gamma}, \tfrac{1}{\epsilon}, \tfrac{1}{\delta})$ sample complexity.
\end{proposition}
\begin{proof}[Proof Sketch] (See full proof in Appendix~\ref{app:value_profile}.) 
	If every $s \in \Scal$ has a unique value profile, the state is perfectly decodable and thus one can simply compute the optimal policy of the certainty-equivalent model. If a set of states share exactly the same value profile---and w.l.o.g.~let's consider 2 states, $s_1$ and $s_2$---realizability implies that $Q^\star(s_1, a) = Q^\star(s_2, a)$, $\forall a\in\Acal$. Now consider the algorithm that treat all states with the same value profile as the same state, which essentially uses a state abstraction that is \emph{$Q^\star$-irrelevant} \citep{li2006towards}. It is known that certainty-equivalence with $Q^\star$-irrelevant abstraction is consistent and enjoys polynomial sample complexity when each state-action pair receives enough data \citep{li2009unifying, hutter2014extreme, jiang2015abstraction, abel2016near, nan_abstraction_notes}. 
\end{proof}

Given that we fail to obtain the lower bound, a conjecture is made below and we hope to resolve it in future work.

\begin{conjecture} \label{conj:complete}
There exists a family of MDPs $\Mcal$ that share the same $\Scal$, $\Acal$, and $\gamma$, such that: any algorithm with $\Fcal = \{Q_{M'}^\star: M'\in\Mcal\}$ as input that can only access states via value profiles cannot have poly($\tfrac{1}{1-\gamma}, C, \ln|\Fcal|, \tfrac{1}{\epsilon}, \tfrac{1}{\delta}$) sample complexity. 
\end{conjecture}



\subsection{Connection to Bisimulation} \label{sec:bisimulation}
As the last piece of technical result of this paper, we show that when $\Fcal$ is a space of piece-wise constant functions under a partition induced by state abstraction $\phi$, the notion of completeness (Assumption~\ref{asm:completeness}, $\Fcal=\Gcal$) is exactly equivalent to a long-studied type of abstractions, known as \emph{bisimulation} \cite{whitt1978approximations, even-dar2003approximate, ravindran2004algebraic, li2006towards}.
\begin{definition}[Bisimulation] \label{def:bisimulation}
	An abstraction $\phi: \Scal\to\Scal_\phi$ is a bisimulation in an MDP $M$, if  $\forall s_1, s_2$ where $\phi(s_1) = \phi(s_2)$ (i.e., they are aggregated), $R(s_1, a)= R(s_2, a)$ and $\sum_{s\in\phi^{-1}(x)}P(s|s_1, a) = \sum_{s\in\phi^{-1}(x)}P(s|s_2,a)$ for all $a\in\Acal$, $x \in \Scal_\phi$. 
\end{definition}

\begin{definition}[Piece-wise constant function class] \label{def:pwconst}
Given an abstraction $\phi$, define $\Fcal^\phi \subset (\Scal\times\Acal \to [0, \Vmax])$ as the set of all functions $f$ that are piece-wise constant under $\phi$. That is, $\forall s_1,s_2\in\Scal$ where $\phi(s_1)=\phi(s_2),$ we have $f(s_1,a)=f(s_2,a)$, $\forall a\in\Acal$ .
\end{definition}

\begin{proposition} \label{prop:bisimulation}
	$\phi$ is bisimulation $\Leftrightarrow$ $\Fcal^\phi$ satisfies completeness (Assumption~\ref{asm:completeness} with $\Fcal=\Gcal=\Fcal^\phi$).
\end{proposition}
The ``$\Rightarrow$'' part is trivial, but the ``$\Leftarrow$'' part is less obvious. The proof shows that if $\phi$ is not a bisimulation, we can find $f\in\Fcal^\phi$ either to witness the reward error or the transition error, and in the latter case, the choice of $f$ achieves the maximum discrepancy in an integral probability metric \citep{muller1997integral} interpretation of  the bisimulation condition on transition dynamics. Details are provided in Appendix~\ref{app:abstraction}, where we prove a stronger result that relates the approximation error of bisimulation to the violation of completeness.

\section{Discussions and Related Work} \label{sec:conclusion}

In this paper, we examine the common assumptions that enable finite sample guarantees for value-function approximation methods. Concretely, we provide an information-theoretic lower bound in  Section~\ref{sec:con_lower}, showing that not constraining the concentratability coefficient $C$  immediately precludes sample-efficient learning even with benign data. We also introduce a general family of problems of interest in empirical RL that yield low concentratability (Section~\ref{sec:visgrid}). 

In comparison, the necessity of completeness is still a mystery, and our investigation in Section~\ref{sec:lower_complete} mostly shows the highly nontrivial nature of the lower bound (assuming it exists) as we eliminate two general proof styles. We hope these negative results can guide the search for novel constructions that reflect the fundamental difficulties of reinforcement learning in the function approximation setting. 

We conclude the paper with some discussions. 

\para{Alternative assumptions to completeness}
As we note in Section~\ref{sec:lower_complete}, even if Conjecture~\ref{conj:complete} is true, it would not imply that completeness is absolutely necessary, as other assumptions may also break the lower bound.   Furthermore, additional assumptions are not necessarily made on the value-function class (e.g., that $\That$ being a contraction \citep{gordon1995stable, szepesvari2004interpolation, lizotte2011convergent, pires2016policy}), and can instead take the form of requiring another function class to realize other objects of interest, such as state distributions \citep{chen2018scalable, liu2018breaking}. Regardless, all of these approaches face the same fundamental question on the necessity of the additional/stronger assumptions being made, to which our Conjecture~\ref{conj:complete} is an important piece if not the final answer. We hope to resolve this important open question in the future.

\para{Related work that has not been covered} The conjectured insufficiency of realizability (Conjecture~\ref{conj:complete}) is related to various undesirable phenomena in learning with bootstrapped targets, which has been of constant interest to RL researchers \citep{sutton2015introduction, van2018deep, lu2018non}. As far as we know,  all existing efforts that investigate this issue are algorithm-specific (apart from \citet[Section 11.6]{sutton2018reinforcement} and the references therein, which has been discussed in Section~\ref{sec:lower_complete}), and our information-theoretic perspective is novel.

\para{Relaxation of Assumption~\ref{asm:concentratability} using the structure of $\Fcal$}
The concentratability coefficient $C$ is defined as a function of the MDP, even in its most complicated version \cite{farahmand2010error}. In Section~\ref{sec:visgrid} we discover a connection to \emph{Bellman rank} \citep{jiang2017contextual}, which can be viewed as its counterpart for online exploration. Interestingly, Bellman rank depends both on the environmental dynamics \emph{and the function class $\Fcal$}, and in some cases, the latter dependence is crucial to obtaining low-rankness (e.g., for Linear Quadratic Regulators; see their Proposition 5). Similarly, we may improve the definition of concentratability and make it more widely applicable by incorporating $\Fcal$ into the definition. In Appendix~\ref{app:relax}, we discuss some preliminary ideas based on the theoretical results in this paper. 

\section*{Acknowledgements}
We gratefully thank the constructive comments from Alekh Agarwal and Anonymous Reviewer \#3.

\bibliography{bib}
\bibliographystyle{icml2019}

\onecolumn
\appendix
\section{Proof of Proposition~\ref{prop:fqi_minimax}}
\label{app:minimax_opt}
\paragraph{Claim 1:} Since $\Gcal=\Fcal$, we have that, $\forall f\in\Fcal$,
$$\argmax_{g\in\Gcal}\left(\Lcal_D(f;f)-\Lcal_D(g;f)\right)=\argmax_{g\in\Gcal}-\Lcal_D(g;f)=\argmin_{g\in\Gcal}\Lcal_D(g;f)=\widehat\Tcal_{\Gcal}f=\widehat\Tcal_{\Fcal}f.$$
	
Therefore, $\Lcal_{D}(\hat g; \hat f) = \Lcal_D(\widehat\Tcal_{\Fcal}\hat f; \hat f)$, and the condition $\Lcal_D(\hat f;\hat f)-\Lcal_D(\hat g;\hat f)=0$ gives us that $\Lcal_D(\hat f;\hat f)-\Lcal_D(\widehat\Tcal_{\Fcal}\hat f;\hat f)=0$. From the definition, we know that $\widehat\Tcal_{\Fcal}\hat f=\argmin_{f \in\Fcal}\Lcal_D(f;\hat f)$. Hence $\Lcal_D(\hat f; \hat f)=\Lcal_D(\widehat\Tcal_{\Fcal}\hat f;\hat f)=\min_{f\in\Fcal}\Lcal_D(f;\hat f)$, which means $\hat f=\argmin_{f \in\Fcal}\Lcal_D(f;\hat f)$ and $\hat f$ is a fixed point for FQI.

\paragraph{Claim 2:} Since $\Gcal=\Fcal$, for any $f\in\Fcal$, we can always choose $g=f$. Therefore, for any $f\in\Fcal$, $\sup_{g\in\Gcal} \Lcal_D(f;f)-\Lcal_D(g;f)\ge 0$, which further means $\inf_{f \in\Fcal}\sup_{g\in\Gcal} \Lcal_D(f;f)-\Lcal_D(g;f)\ge 0$, and the value of the optimization problem is non-negative. If we have $f_k=f_{k-1}$ for some $k$ in FQI, then we know that $f_{k-1}\in\Fcal$, $f_k\in\Fcal=\Gcal$ and $\Lcal_D(f_{k-1};f_{k-1})-\Lcal_D(f_k;f_{k-1})=0$. This tells us that $f=f_{k-1}$ and $g=f_k$ achieve the optimal value, so $\hat f=\hat g=f_k$ is a solution to Eq.(\ref{eq:minimax}).

\paragraph{Claim 3:} Prove by contradiction. If FQI does not oscillate and a fixed point of FQI is $f_{k-1}(=f_k)$, the previous result gives us that $\hat f=\hat g=f_k$ is a solution to Eq.\eqref{eq:minimax}, with the minimax objective value being $\Lcal_D(f_k;f_k)-\Lcal_D( f_k; f_k)=0$. Contradiction.


\section{Example of Low Concentratability in Rich-Observation Problems} \label{app:visgrid}

\begin{definition}[Reactive POMDPs \citep{jiang2017contextual}]
A reactive POMDP is a decision process specified by a  finite hidden state space $\Zcal$, an (arbitrarily large) observation space $\Scal$, an action space $\Acal$, hidden state dynamics $\Gamma: \Zcal \times \Acal \to \Delta(\Zcal)$, an initial hidden state distribution $\Gamma_1 \in \Delta(\Zcal)$, an emission process $P: \Zcal\to \Delta(\Scal)$, a reward function $R: \Xcal\times\Acal \to \Delta([0,1])$, and a discount factor $\gamma \in [0, 1)$. A trajectory is generated as $z_1 \sim \Gamma_1$,  $s_1 \sim P(\cdot |z_1)$, $r_1 \sim R(s_1, a_1)$, $z_2 \sim \Gamma(z_1, a_1)$, $s_2 \sim P(\cdot | z_2),\ldots$, where the hidden states $z_h$'s are not observable to the agent. Moreover, the $Q^\star$ function of this POMDP is assumed to only depend on the last observation $s_h$, hence  ``reactive'' POMDPs. We make a further simplification by assuming that the observations are indeed Markov (which implies reactive $Q^\star$). 
\end{definition}

\begin{proposition}[Formal version of Proposition~\ref{prop:visgrid}]
	Let the environment be a reactive POMDP as defined above, where the underlying hidden state space $\Zcal$ is finite. The (Markov) observation space $\Scal$ is finite but can be arbitrarily large. Assume that the number of admissible distributions is finite,\footnote{This assumption is only introduced to get around of some technical subtleties, and the resulting upper bound on $C$ has no dependence on the number of admissible distributions.} there exists a distribution $\mu_{\Scal} \in \Delta(\Scal)$ that can be expressed as a mixture of admissible distributions (more accurately, their marginals over states), such that $C \le |\Zcal \times \Acal|$ when $\mu := \mu_{\Scal} \times \textrm{Unif}(\Acal)$ is used as the data distribution (recall the definition of $C$ in Assumption~\ref{asm:concentratability}).
\end{proposition}

\begin{proof}
	The proof contains two parts: the first part shows that a certain matrix consisting of admissible distributions has low rank, and the second part exploits the low-rankness to construct the mixture distribution described in the proposition statement and shows that it yields low concentratability coefficient $C$. 

	By definition, an admissible (state-action) distribution takes the form of $\eta_h^\pi \in \Delta(\Scal\times\Acal)$, that is the distribution over state-action pairs induced by rolling into time step $h$ with policy $\pi$. Let  $\nu_h^\pi(s)$ denote the corresponding marginal probability over states, which we call an \emph{admissible state distribution}. Note that $\eta_h^\pi(s,a)=\nu_h^\pi(s)\pi(a|s)$.
	
	Let there be a total of $N$ admissible state distributions (we assumed $N$ to be finite). Order them in an arbitrary manner and let the $i$-th admissible state distribution be $\nu_{h_i}^{\pi_i}$, for $i=1, \ldots, N$. Stacking these distributions as a matrix: 
	\begin{equation*}
	A_\Scal :=
	\begin{bmatrix}
	&\nu_{h_1}^{\pi_1}(s^1)&\cdots&\nu_{h_1}^{\pi_1}(s^{|\Scal|})\\
	&\vdots&\ddots&\vdots\\
	&\nu_{h_N}^{\pi_N}(s^1)&\cdots&\nu_{h_N}^{\pi_N}(s^{|\Scal|})
	\end{bmatrix},
	\end{equation*}
	where each row is indexed by an admissible state distribution and each column is indexed by a state, and $\Scal :=\{s^1,\cdots,s^{|\Scal|}\}$.
	
	In reactive POMDPs, we can also define admissible distributions over hidden states $\Zcal$. For any $z\in\Zcal$ and $a\in\Acal$, with abuse of notation, we use $\eta_h^\pi(z,a)$ and $\nu_h^\pi(z)$ to denote the distribution over hidden states (and actions) at step $h$ induced by $\pi$. For any $\pi$, $h$, and $s$, the distribution over observations can be decomposed as $\nu_h^\pi(s)=\sum_{z\in\Zcal} P(s|z) \nu_h^{\pi}(z)$, where $P(s|z)$ is the emission process and is independent of the policy or the timestep. Therefore, we have
	\begin{equation*}
	A_\Scal =
	\begin{bmatrix}
	&\nu^{\pi_1}_{h_1}(z^1)&\cdots&\nu^{\pi_1}_{h_1}(z^{|\Zcal|})\\
	&\vdots&\ddots&\vdots\\
	&\nu^{\pi_N}_{h_N}(z^1)&\cdots&\nu^{\pi_N}_{h_N}(z^{|\Zcal|})
	\end{bmatrix}
	\begin{bmatrix}
	&P(s^1|z^1)&\cdots& P(s^{|\Scal|}|z^1)\\
	&\vdots&\ddots&\vdots\\
	&P(s^1|z^{|\Zcal|})&\cdots& P(s^{|\Scal|}|z^{|\Zcal|})
	\end{bmatrix}
	:=A_{\Zcal}\, P_{\Scal|\Zcal}.
	\end{equation*}
	From the above, we conclude that $r:=\text{rank}(A_\Zcal)\le |\Zcal|$. 
	
	In the rest of the proof we describe how to construct the mixture distribution and show that it yields low concentratability coefficient $C$. First, we factorize $A_{\Zcal}$ as the product of two matrices with full column rank and full row rank, respectively: 
	$$
	A_{\Zcal} = B_{\Zcal} \, C_{\Zcal}.
	$$
	We know that $\text{rank}(B_\Zcal) = \text{rank}(A_\Zcal) = r \le |\Zcal|$. 
	
	Now let's focus on $B_{\Zcal} := \begin{bmatrix} b_1 &\cdots & b_M \end{bmatrix}^\top$, where $b_i^\top$ is its $i$-th row. Let $D_{\Zcal}$ consists of $r$ rows from $B_{\Zcal}$ that maximize the absolute value of determinant (i.e., the spanned volume). That is
	\begin{equation*}
	D_{\Zcal}:=\begin{bmatrix}
	b_{i_1}^\top\\
	\vdots\\
	b_{i_r}^\top
	\end{bmatrix}, \quad \text{where} \quad
	(i_1,\ldots,i_r):=\argmax_{i_1',\ldots,i_r'\in\{1,\cdots,N\}} \left|\det
	\begin{bmatrix}
	b_{i_1'}^\top\\
	\vdots\\
	b_{i_r'}^\top
	\end{bmatrix}\right|.
	\end{equation*}
	
	Since $D_\Zcal$ maximizes the absolute value of the determinant, $|\det D_\Zcal|>0$ and $D_\Zcal$ is a full-rank square matrix. As a result, any row $b_i^\top$ of $B_\Zcal$ is a linear combination of rows in $D_\Zcal$. So there exists $\alpha_1,\ldots,\alpha_r\in\RR$, such that $b_i=\sum_{j=1}^r \alpha_j d_j$ where $d_j := b_{i_j}$. We claim that $|\alpha_j|\le1$ always holds. 
	
	This can be proved by contradiction. Assume that $|\alpha_{j_0}|>1$, then consider the matrix 
	\begin{equation*}
	E_{\Zcal}=\begin{bmatrix}
	d_1^\top\\
	\vdots\\
	d_{j_0-1}^\top\\
	b_i^\top\\
	d_{j_0+1}^\top\\
	\vdots\\
	d_r^\top\\
	\end{bmatrix}=\begin{bmatrix}
	1  & \cdots & 0 & 0 & 0 & \cdots & 0\\
	\vdots & \ddots & \vdots & \vdots & \vdots & \ddots & \vdots\\
	0  & \cdots & 1 & 0 & 0 & \cdots & 0\\
	\alpha_1  & \cdots & \alpha_{j_0-1} & \alpha_{j_0} & \alpha_{j_0+1} & \cdots & \alpha_r\\
	0  & \cdots & 0 & 0 & 1 & \cdots & 0\\
	\vdots  & \ddots & \vdots & \vdots & \vdots & \ddots & \vdots\\
	0  & \cdots & 0 & 0 & 0 & \cdots & 1\\
	\end{bmatrix}D_\Zcal:=T_\Zcal D_\Zcal.
	\end{equation*}
	This matrix essentially replaces the $j_0$-th row of $D_\Zcal$ with $b_i^\top$. Since $D_{\Zcal}$ is volume maximizing, the volume of $E_\Zcal$ should not increase. 
	Calculating the determinant, however, we get $|\det E_\Zcal|= |\det T_\Zcal \, \det D_\Zcal|=|\alpha_{j_0}||\det D_\Zcal|>|\det D_\Zcal|$, which causes a contradiction. 
	
	Finally, we construct the data distribution as a mixture of admissible distributions. Let $\mu(s)=\frac{1}{r}\sum_{j=1}^r \nu_{h_{i_j}}^{\pi_{i_j}}(s)$ and $\mu(a|s)=1/|\Acal|$. It is easy to check that $\mu(s,a)$ is a valid distribution. Then for any $i\in\{1,\cdots,N\}$, 
	$$\frac{\nu^{\pi_i}_{h_i}(s,a)}{\mu(s,a)}=\frac{\nu^{\pi_i}_{h_i}(s)}{\mu(s)}\frac{\nu(a|s)}{\mu(a|s)}\le \frac{\nu^{\pi_i}_{h_i}(s)}{\mu(s)}\frac{1}{1/|\Acal|}.$$
	Now recall that for any $i$, there exists $|\alpha_j|\le 1, j=1,\ldots,r$, such that $b_i=\sum_{j=1}^r \alpha_j b_{i_j}$. Since $A_\Scal=B_\Zcal C_\Zcal P_{\Scal|\Zcal}$, comparing the $i$-th row of both sides, we have 
	$$\nu^{\pi_i}_{h_i}(s)=\sum_{j=1}^r\alpha_j b_{i_j}^\top C_\Zcal P_{\Scal|\Zcal}= \sum_{j=1}^r\alpha_j\nu^{\pi_{i_j}}_{h_{i_j}}(s)\le \sum_{j=1}^r|\alpha_j|\nu^{\pi_{i_j}}_{h_{i_j}}(s).$$
	The inequality follows from the non-negativity of probabilities. Hence,
	\begin{align*}
	\frac{\nu^{\pi_i}_{h_i}(s,a)}{\mu(s,a)}\le\frac{\sum_{j=1}^r|\alpha_j|\nu^{\pi_{i_j}}_{h_{i_j}}(s)}{\frac{1}{r}\sum_{j=1}^r \nu^{\pi_{i_j}}_{h_{i_j}}(s)}|\Acal|\le \frac{\sum_{j=1}^r\nu^{\pi_{i_j}}_{h_{i_j}}(s)}{\frac{1}{r}\sum_{j=1}^r \nu^{\pi_{i_j}}_{h_{i_j}}(s)}|\Acal|=r|\Acal|\le |\Zcal\times\Acal|. \tag*{\qedhere}
	\end{align*}
\end{proof}
\section{Analysis of FQI} 
\label{app:fqi}
We state the more general error bound for FQI when Assumption~\ref{asm:completeness} only holds approximately; Theorem~\ref{thm:fqi} is a direct corollary of this result. Note that although our bound contains a slow-rate term ($n^{-1/4}$), it is multiplied by $\sqrt[4]{\epsilon_{\Fcal, \Fcal}}$ and becomes small when $\epsilon_{\Fcal,\Fcal}$ is small. Furthermore, a closer examination of the bound reveals that the slow-rate term is \emph{always} a geometric mean of the fast-rate term and the approximation error term, so the slow-rate term never dominates the bound. The bound for the minimax algorithm (Theorem~\ref{thm:minimax_full}) is in a similar situation, which distinguishes our bound from prior results for this algorithm that contains a ``real'' and dominating slow-rate term \citep{antos2008learning, munos2008finite}.

\begin{theorem}[Error bound for FQI] \label{thm:fqi_full}
	Given a dataset $D=\{(s,a,r,s')\}$ with sample size $|D|=n$, $\Fcal$ that satisfies approximate completeness (Assumption~\ref{asm:completeness}) with error $\epsilon_{\Fcal, \Fcal}$, with probability at least $1-\delta$, the output policy of FQI after $k$ iterations, $\pi_{f_k}$, satisfies\footnote{Big-Oh notations in this paper only suppress absolute constants.}
	$$v^\star - v^{\pi_{\fh}} \le O\left(\frac{\Vmax}{(1-\gamma)^2}\left(\sqrt{\frac{C\ln\frac{|\Fcal|}{\delta}}{n}}+\sqrt[4]{\frac{C\ln\frac{|\Fcal|}{\delta}}{n}\epsilon_{\Fcal,\Fcal}}\right)\right)+\frac{2(\sqrt{C\epsilon_{\Fcal,\Fcal}}+\gamma^k(1-\gamma)\Vmax)}{(1-\gamma)^2}.$$
\end{theorem}

To prove the theorem, we first define some useful notations for the proof and prove a few helper lemmas. Some of these notations/lemmas will also be helpful for the later analysis of the minimax algorithm, and we will reuse them.

\paragraph{Additional Notations}
We use $\eta_h^\pi \times \pi'$ to denote the joint distribution over $(s,a)$, where $s \sim \eta_h^\pi$ and $a \sim \pi'(s)$. For any $\nu\in\Delta(\Scal\times\Acal)$, define $P(\nu)$ as a  distribution  over states such that
$s' \sim P(\nu) \Leftrightarrow (s,a) \sim \nu$,~ $s'\sim P(s,a)$.

The first lemma is the direct consequence of concentratability (recall Assumption~\ref{asm:concentratability}).
\begin{lemma}
Let $\mu$ be any admissible distribution. $\|\cdot\|_{2,\nu} \le \sqrt{C} \|\cdot\|_{2,\mu}$.
\end{lemma}
\begin{proof}

For any function $g:\Scal\times\Acal\rightarrow\RR$, we have
\begin{align*}
	\|g\|_{2, \nu} &=\left( \sum_{(s,a)\in\Scal\times\Acal}\left|g(s,a)\right|^2\nu(s,a)\right)^{1/2}\\
	&\le \left( \sum_{(s,a)\in\Scal\times\Acal}\left|g(s,a)\right|^2C\mu(s,a)\right)^{1/2}\\
	&=\sqrt{C}\left( \sum_{(s,a)\in\Scal\times\Acal}\left|g(s,a)\right|^2\mu(s,a)\right)^{1/2}=\sqrt{C}\|g\|_{2,\mu}. \tag*{\qedhere}
\end{align*}
\end{proof}

The next lemma relates the suboptimality of a policy greedy w.r.t.~a function $f$ to $\|f - Q^\star\|$.
\begin{lemma}\label{lem:decompose}
	Let $f:\Scal\times\Acal\rightarrow\mathbb{R}$ and $\hat \pi=\pi_f$ be the policy of interest, we have
	$$v^\star - v^{\hat\pi}\leq \sum_{h=1}^\infty \gamma^{h-1} \left(\|Q^\star- f\|_{2,\eta^{\hat \pi}_h \times \pi^\star} + \|Q^\star - f\|_{2,\eta^{\hat \pi}_h \times \hat \pi}\right). $$
\end{lemma}
\begin{proof}
	\begin{align*}
	v^\star - v^{\hat\pi} 
	= &~ \sum_{h=1}^\infty \gamma^{h-1} \EE_{s\sim \eta^{\hat\pi}_h} [V^\star(s) - Q^\star(s,\hat{\pi})] \tag{see e.g., \citet[Lemma 6.1]{kakade2002approximately}} \\
	\le &~ \sum_{h=1}^\infty \gamma^{h-1} \EE_{s\sim \eta^{\hat \pi}_h} [Q^\star(s, \pi^\star) - f(s, \pi^\star) + f(s, \hat\pi) - Q^\star(s,\hat\pi)] \nonumber\\
	\le &~ \sum_{h=1}^\infty \gamma^{h-1} \left(\|Q^\star- f\|_{1, \eta^{\hat \pi}_h \times \pi^\star} + \|Q^\star - f\|_{1, \eta^{\hat \pi}_h \times \hat \pi}\right) \nonumber\\
	\le &~ \sum_{h=1}^\infty \gamma^{h-1} \left(\|Q^\star- f\|_{2,\eta^{\hat \pi}_h \times \pi^\star} + \|Q^\star - f\|_{2,\eta^{\hat \pi}_h \times \hat \pi}\right). \tag*{\qedhere}
	\end{align*}
\end{proof}

The following lemma, vaguely speaking, shows that $\max$ operator is a non-expansion in the function approximation setting.
\begin{lemma} \label{lem:piff}
	Assume $f,f':\Scal\times\Acal\rightarrow\mathbb{R}$ and define $\pi_{f,f'}(s):= \argmax_{a\in\Acal} \max\{f(s, a), f'(s, a)\}$. Then we have $\forall \nu \in \Delta(\Scal\times\Acal)$,
		$$
		\|V_f - V_{f'}\|_{2,P(\nu)} \le \|f - f'\|_{2,P(\nu)\times \pi_{f,f'}}.
		$$
\end{lemma}
\begin{proof}
	\begin{align*}
	\|V_f - V_{f'}\|_{2,P(\nu)}^2 
	= &~ \sum_{(s,a)\in\Scal\times\Acal}\sum_{s'\in\Scal}P(s'|s,a) (\max_{a\in\Acal} f(s',a) - \max_{a'\in\Acal} f'(s',a'))^2  \\
	\le &~ \sum_{(s,a)\in\Scal\times\Acal}\sum_{s'\in\Scal}P(s'|s,a) (f(s',\pi_{f,f'}) - f'(s',\pi_{f,f'}))^2 
	= \|f - f'\|_{2,P(\nu) \times \pi_{f,f'}}^2. \tag*{\qedhere}
	\end{align*}
\end{proof}

With the help of Lemma~\ref{lem:piff}, we are able to upper bound $\|f-Q^\star\|$ using the Bellman error $\|f - \Tcal f\|$ under $\ell_2$ norm. The more coarse-grained version w.r.t.~$\ell_\infty$ norm has been proved by \citet{singh1994upper}.
\begin{lemma}\label{lem:iteration}
For an exploratory distribution $\mu\in\Delta(\Scal\times\Acal)$, any distribution $\nu\in\Delta(\Scal\times\Acal)$, policy $\pi$, and $f,f':\Scal\times\Acal\rightarrow\mathbb{R}$, we have 
$$\|f - Q^\star\|_{2,\nu}\leq \sqrt{C}~ \|f - \Tcal f'\|_{2,\muu} + \gamma \| f' - Q^\star\|_{2,P(\nu) \times \pi_{f', Q^\star}}$$
and
$$\|f - Q^\star\|_{2,\nu}\leq \frac{\sqrt{C}}{1-\gamma}~ \|f - \Tcal f\|_{2,\muu}.$$
\end{lemma}

\begin{proof}
For any fixed distribution $\nu$, we have
\begin{align*}
	\|f - Q^\star\|_{2,\nu} 
	= &~ \|f - \Tcal f' + \Tcal f' - Q^\star\|_{2,\nu} \\
	\le &~ \|f - \Tcal f'\|_{2,\nu} + \|\Tcal f' - \Tcal Q^\star\|_{2,\nu} \\
	\le &~ \sqrt{C}~ \|f - \Tcal f'\|_{2,\muu} + \gamma \| V_{f'} - V^\star\|_{2,P(\nu)} \tag{*}\\
	\le &~ \sqrt{C}~ \|f - \Tcal f'\|_{2,\muu} + \gamma \| f' - Q^\star\|_{2,P(\nu) \times \pi_{f', Q^\star}}. \tag{Lemma~\ref{lem:piff}}
\end{align*}
Step (*) holds because:
\begin{align*}
	\|\Tcal f' - \Tcal Q^\star\|_{2,\nu}^2
	= &~ \EE_{(s,a) \sim \nu}\left[\left((\Tcal f')(s,a) - (\Tcal Q^\star)(s,a)\right)^2\right] \\
	= &~ \EE_{(s,a) \sim \nu}\left[\left(\gamma \EE_{s' \sim P(s,a)}[V_{f'}(s') - V^\star(s')] \right)^2\right] \\
	\le &~ \gamma^2 \, \EE_{(s,a) \sim \nu, s' \sim P(s,a)}\left[\left(V_{f'}(s') - V^\star(s') \right)^2\right] \tag{Jensen} \\
	= &~ \gamma^2 \, \EE_{s' \sim P(\nu)}\left[\left(V_{f'}(s') - V^\star(s') \right)^2\right]\\
	= &~\gamma^2 \, \|V_{f'} - V^\star\|_{2,P(\nu)}^2.
\end{align*}

For the second term, let $f'=f$ and $\nu_0=\argmax_{\nu}\|f-Q^\star\|_{2,\nu}$, then we have 
\begin{align*}
\|f-Q^\star\|_{2,\nu_0}\le&~\sqrt{C}~ \|f - \Tcal f\|_{2,\muu} + \gamma \| f - Q^\star\|_{2,P(\nu_0) \times \pi_{f, Q^\star}}.\\
\le&~\sqrt{C}~ \|f - \Tcal f\|_{2,\muu} + \gamma \| f - Q^\star\|_{2,\nu_0}	
\end{align*}

Therefore, $\|f-Q^\star\|_{2,\nu}\le\|f-Q^\star\|_{2,\nu_0}\leq\frac{\sqrt{C}}{1-\gamma}~ \|f - \Tcal f\|_{2,\muu}$.
\end{proof}

Finally, a concentration result that yields fast rate when completeness holds.
\begin{lemma}
\label{lem:fast_1} 
Given the MDP $M = (\Scal, \Acal, P, R, \gamma, \eta_1)$, we assume that the Q-function classes $\Fcal$ and $\Gcal$ are finite but can be exponentially large.  $\Gcal$ approximately realizes $\Tcal \Fcal$ ($\forall f\in\Fcal$, let $g_f^\star=\argmin_{g\in\Gcal}\|g-\Tcal f\|_{2,\mu}$, then $\|g_f^\star-\Tcal f\|_{2,\mu}^2\le\epsilon_{\Fcal,\Gcal}$). The dataset $D$ is generated from $M$ as follows: $(s,a) \sim \muu$, $r = R(s,a)$, $s' \sim P(s,a)$. We have that $\forall f\in\Fcal$, with probability at least $1-\delta$,
\begin{align*}
\Lcal_\muu(\widehat \Tcal_\Gcal f;f)-\Lcal_\muu(g_f^\star;f) \le \frac{56\Vmax^2\ln\frac{|\Fcal| |\Gcal|}{\delta}}{3n}+\sqrt{\frac{32\Vmax^2\ln\frac{|\Fcal| |\Gcal|}{\delta}}{n}\epsilon_{\Fcal,\Gcal}}.
\end{align*}
\end{lemma}
\begin{proof}	
	Fix $f\in\Fcal$ and $g\in\Gcal$, define $$ X(g,f,g_f^\star) := \left(g(s,a) - r - \gamma V_{f}(s')\right)^2 - \left(g_f^\star(s,a) - r - \gamma V_{f}(s')\right)^2.$$
	Plugging each $(s,a,r,s') \in D$ into $X(g,f,g_f^\star)$, we get i.i.d.~variables $X_1(g,f,g_f^\star), X_2(g,f,g_f^\star), \ldots,$ $X_n(g,f,g_f^\star)$. It is easy to see that 
	$$
	\frac{1}{n}\sum_{i=1}^n X_i(g,f,g_f^\star) = \Lcal_{D}(g;f) - \Lcal_{D}(g_f^\star;f).
	$$
	Then we bound variance of $X$:
	\begin{align}
	\VV[X(g,f,g_f^\star)] \le &~ \EE[X(g,f,g_f^\star)^2] \notag\\
	= &~ \EE \left[\left(\big(g(s,a) - r - \gamma V_{f}(s')\big)^2 - \big(g_f^\star(s,a) - r - \gamma  V_{f}(s')\big)^2\right)^2\right] \notag\\
	= &~ \EE \left[\big(g(s,a)  - g_f^\star(s,a) \big)^2 \big(g(s,a) + g_f^\star(s,a) - 2r - 2\gamma V_{f}(s') \big)^2\right] \notag\\
	\le &~ 4\Vmax^2~ \EE \left[\big(g(s,a)  - g_f^\star(s,a) \big)^2\right] \notag\\
	= &~ 4\Vmax^2~ \|g - g_f^\star\|_{2,\muu}^2 \label{eq:bound_var}\\
	\le &~ 8\Vmax^2~ (\EE[X(g,f,g_f^\star)]+2\epsilon_{\Fcal,\Gcal}) \tag{*}.
	\end{align}
	Step (*) holds because 
	\begin{align*}
	&~\|g-g_f^\star\|_{2,\muu}^2\\
	\leq&~2\left(\|g-\Tcal f\|_{2,\muu}^2+\|\Tcal f-g_f^\star\|_{2,\muu}^2\right) \tag{$(a+b)^2\leq 2a^2+2b^2$} \\
	\leq&~2\left(\|g-\Tcal f\|_{2,\muu}^2-\|\Tcal f-g_f^\star\|_{2,\muu}^2+2\|\Tcal f-g_f^\star\|_{2,\muu}^2\right)\\
	=&~2\left[(\Lcal_{\muu}(g;f) - \Lcal_{\muu}(\Tcal f; f))-(\Lcal_{\muu}(g_f^\star;f) - \Lcal_{\muu}(\Tcal f; f))+2\|\Tcal f-g_f^\star\|_{2,\muu}^2\right]\\
	=&~2\left(\EE[X(g,f,g_f^\star)]+2\|\Tcal f-g_f^\star\|_{2,\muu}^2\right)\\
	\leq &~2(\EE\left[X(g,f,g_f^\star)\right]+2\epsilon_{\Fcal,\Gcal})
	\end{align*}

	Next, we apply (one-sided) Bernstein's inequality and union bound over all $f\in\Fcal$ and $g\in\Gcal$. With probability at least $1-\delta$, we have
	\begin{align}
	&~ \EE[X(g, f,g_f^\star)] - \frac{1}{n}\sum_{i=1}^n X_i(f,f,g_f^\star)\notag\\
	\le &~ \sqrt{\frac{2 \VV[X(g,f,g_f^\star)] \ln\tfrac{|\Fcal| |\Gcal|}{\delta}}{n}} + \frac{4\Vmax^2 \ln\tfrac{|\Fcal| |\Gcal|}{\delta}}{3n} \notag\\
	= &~ \sqrt{\frac{16 \Vmax^2 \left(\EE[X(g, f,g_f^\star)]+2\epsilon_{\Fcal,\Gcal}\right) \ln\tfrac{|\Fcal| |\Gcal|}{\delta}}{n}} + \frac{4\Vmax^2 \ln\tfrac{|\Fcal| |\Gcal|}{\delta}}{3n} \label{eq:Bernstein}.
	\end{align}
	Since $\widehat\Tcal_\Gcal f$ minimizes $\Lcal_D(\holder; f)$, it also minimizes $\frac{1}{n}\sum_{i=1}^n X_i(\cdot,f,g_f^\star)$. This is because the two objectives only differ by a constant $\Lcal_D(g_f^\star; f)$. Hence,
	$$
	\frac{1}{n}\sum_{i=1}^n X_i(\widehat \Tcal_\Gcal f, f,g_f^\star) \le 
	\frac{1}{n}\sum_{i=1}^n X_i(g_f^\star, f,g_f^\star) = 0.
	$$
	Then,
	\begin{align*}
	\EE[X(\widehat\Tcal_\Gcal f,f,g_f^\star)]\le \sqrt{\frac{16 \Vmax^2 \left(\EE[X(\widehat\Tcal_\Gcal f, f,g_f^\star)]+2\epsilon_{\Fcal,\Gcal}\right) \ln\tfrac{|\Fcal| |\Gcal|}{\delta}}{n}} + \frac{4\Vmax^2 \ln\tfrac{|\Fcal| |\Gcal|}{\delta}}{3n}.
	\end{align*}
	Solving for the quadratic formula,
	\begin{align*}
	\EE[X(\widehat\Tcal_\Gcal f, f,g_f^\star)]\le&~ \sqrt{48\left(\frac{4\Vmax^2\ln\frac{|\Fcal| |\Gcal|}{\delta}}{3n}\right)^2+
		\frac{32\Vmax^2\ln\frac{|\Fcal| |\Gcal|}{\delta}}{n}\epsilon_{\Fcal,\Gcal}}+\frac{28\Vmax^2\ln\frac{|\Fcal| |\Gcal|}{\delta}}{3n}\\
	\leq&~\frac{(28+16\sqrt{3})\Vmax^2\ln\frac{|\Fcal| |\Gcal|}{\delta}}{3n}+\sqrt{\frac{32\Vmax^2\ln\frac{|\Fcal| |\Gcal|}{\delta}}{n}\epsilon_{\Fcal,\Gcal}} \tag{$\sqrt{a+b}\leq\sqrt{a}+\sqrt{b}$ and $\ln \frac{|\Fcal| |\Gcal|}{\delta}>0$}\\
	\leq&~\frac{56\Vmax^2\ln\frac{|\Fcal| |\Gcal|}{\delta}}{3n}+\sqrt{\frac{32\Vmax^2\ln\frac{|\Fcal| |\Gcal|}{\delta}}{n}\epsilon_{\Fcal,\Gcal}}
	\end{align*}
	Noticing that $\EE [X(\widehat\Tcal_\Gcal f, f,g_f^\star)]=\Lcal_\muu(\widehat\Tcal_\Gcal f;f)-\Lcal_\muu(g_f^\star;f)$, we complete the proof.
\end{proof}

%
%

Now we are ready to prove the main theorem.
\begin{proof}[Proof \textbf{of Theorem~\ref{thm:fqi_full}}]~~
Firstly, we can let $f=f_k$ and $f'=f_{k-1}$ in Lemma \ref{lem:iteration}. This gives us that $$\|f_k - Q^\star\|_{2,\nu}\leq \sqrt{C}~ \|f_k - \Tcal f_{k-1}\|_{2,\muu} + \gamma \| f_{k-1} - Q^\star\|_{2,P(\nu) \times \pi_{f_{k-1}, Q^\star}}.$$

Note that we can apply the same analysis on $P(\nu) \times \pi_{\fo, Q^\star}$ and expand the inequality $k$ times. It then suffices to upper bound $\|\fh - \Tcal \fo\|_{2,\muu}$.
\begin{align*}
	&~\|\fh - \Tcal \fo\|_{2,\muu}^2\\
	= &~ \Lcal_{\muu}(\fh; \fo) - \Lcal_{\muu}(\Tcal \fo; \fo) \tag{$\Lcal$ squared loss + $\Tcal \fo$ Bayes optimal}\\ 
	= &~ [\Lcal_{\muu}(\fh; \fo) -\Lcal_{\muu}(g_{f_{k-1}}^\star; \fo)]+[\Lcal_{\muu}(g_{f_{k-1}}^\star; \fo)- \Lcal_{\muu}(\Tcal \fo; \fo)] \\ 
	\le &~ \epsilon_1 + \|g_{f_{k-1}}^\star-\Tcal f_{k-1}\|_{2,\muu}^2 \tag{Let $\Gcal=\Fcal$ in Lemma \ref{lem:fast_1} + $ \Lcal$ squared loss + $\Tcal \fo$ Bayes optimal}  \\
	\le &~ \epsilon_1 +\epsilon_{\Fcal,\Fcal}. \tag{The selection of $g_{f_{k-1}}^\star$}
\end{align*}
The inequality holds with probability at least $1-\delta$ and $\epsilon_1=\frac{56\Vmax^2\ln\frac{|\Fcal|^2}{\delta}}{3n}+\sqrt{\frac{32\Vmax^2\ln\frac{|\Fcal|^2}{\delta}}{n}\epsilon_{\Fcal,\Fcal}}$.

Noticing that $\epsilon_1$ and $\epsilon_{\Fcal,\Fcal}$ do not depend on $k$, and the inequality holds simultaneously for different $k$, we have that 
$$
\|\fh - Q^\star\|_{2,\nu} \le  \frac{1-\gamma^k}{1-\gamma} \sqrt{C(\epsilon_1+\epsilon_{\Fcal.\Fcal})} + \gamma^k \Vmax.
$$
Applying this to Lemma \ref{lem:decompose}, we have that
\begin{align*}
v^\star - v^{\pi_{\fh}} \le& \frac{2}{1-\gamma}\left(\frac{1-\gamma^k}{1-\gamma} \sqrt{C(\epsilon_1+\epsilon_{\Fcal,\Fcal})} + \gamma^k \Vmax\right)\\
\le& \frac{2}{(1-\gamma)^2}\left( \sqrt{C\epsilon_1} +\sqrt{C\epsilon_{\Fcal,\Fcal}} + \gamma^k (1-\gamma) \Vmax\right)\\
\le& \frac{2}{(1-\gamma)^2}\left( \sqrt{\frac{56C\Vmax^2\ln\frac{|\Fcal|^2}{\delta}}{3n}}+\sqrt[4]{{\frac{32C\Vmax^2\ln\frac{|\Fcal|^2}{\delta}}{n}\epsilon_{\Fcal,\Fcal}}} +\sqrt{C\epsilon_{\Fcal,\Fcal}} + \gamma^k (1-\gamma)\Vmax\right).
\end{align*}
The proof is completed by simplifying the expression.
\end{proof}
\section{Analysis of the Minimax Algorithm}
\label{app:minimax}
We state the more general error bound for the minimax algorithm when Assumptions~\ref{asm:realizability} and \ref{asm:completeness} only hold approximately;  Theorem~\ref{thm:minimax} is a direct corollary of this result. See Appendix~\ref{app:fqi} for the interpretations and discussions of this result.
\begin{theorem}[Error bound for the minimax algorithm] \label{thm:minimax_full}
	Given a dataset $D=\{(s,a,r,s')\}$ with sample size $|D|=n$, $\Fcal$ that satisfies approximate realizability with error $\epsilon_{\Fcal}$, and $\Gcal$ that satisfies approximate completeness with error $\epsilon_{\Fcal, \Gcal}$, with probability at least $1-\delta$, the output policy of the minimax algorithm (Eq.\eqref{eq:minimax}), $\pi_{\hat f}$, satisfies: 
	\begin{align*}
	v^\star - v^{\pi_{\hat f}}
	\le &~O\left(\frac{\Vmax\sqrt{C}}{(1-\gamma)^2}\left(\sqrt{\frac{\ln\tfrac{|\Fcal| |\Gcal|}{\delta}}{n}} + \sqrt[4]{\frac{ \ln\tfrac{|\Fcal| |\Gcal|}{\delta}}{n}(\epsilon_\Fcal+\epsilon_{\Fcal,\Gcal}})\right)\right) +\frac{2\sqrt{2C}}{(1-\gamma)^2}\left(\sqrt{\epsilon_\Fcal}+ \sqrt{2\epsilon_{\Fcal,\Gcal}}\right).
	\end{align*}
\end{theorem}

We provide a sketched outline before diving into the detailed proof:
\begin{enumerate}
	\item The objective in the minimax form is $$\inf_{f\in\Fcal}\sup_{g\in\Gcal}\left(\Lcal_D(f;f)-\Lcal_D(g;f)\right)=\inf_{f\in\Fcal}(\Lcal_D(f;f)-\Lcal_D(\widehat\Tcal_{\Gcal}f;f)).$$
	\item We begin with dropping the dependence on function class $\Gcal$ by upper bounding the difference $|\Lcal_D(\widehat\Tcal_\Gcal f;f)-\Lcal_D(\Tcal f;f)|.$ This is Lemma \ref{lem:fast_2} and can be separated into two substeps.
	
	The first substep is to bound $|\frac{1}{n}\sum_{i=1}^n X_i(\widehat\Tcal_{\Gcal}f,f,g_f^\star)|$, where $X(\widehat\Tcal_{\Gcal}f,f,g_f^\star) = (\widehat\Tcal_{\Gcal}f(s,a) - r - \gamma V_{f}(s'))^2 - (g_f^\star(s,a) - r - \gamma V_{f}(s'))^2$. This error is between the output of the algorithm and the best function in class $\Gcal$.
	
	The second substep is to bound $|\frac{1}{n}\sum_{i=1}^nY_i(g_f^\star,f)|$, where $Y(g_f^\star,f) := (g_f^\star(s,a) - r - \gamma V_{f}(s'))^2 - ((\Tcal f)(s,a) - r - \gamma V_{f}(s'))^2.$ This error is between the best function in class $\Gcal$ and the true Bellman update $\Tcal f$.
	
	In this way, we can change the objective in the minimax form to $\inf_{f\in\Fcal}(\Lcal_D(f;f)-\Lcal_D(\Tcal f;f))$, within a bounded error.
	
	\item Then, we only need to consider the function class $\Fcal$, since $\inf_{f\in\Fcal}(\Lcal_D(f;f)-\Lcal_D(\Tcal f;f))$ is only related to $\Fcal$. The proof can be finished by the following three substeps.
	
	Firstly, by the optimality of  $\hat f$ and the previous error bounds, we can bound the difference between $\frac{1}{n}\sum_{i=1}^nZ_i(\hat f)$ and $\frac{1}{n}\sum_{i=1}^n Z_i(f^\star)$ by $\epsilon_2$ in Lemma \ref{lem:fast_2}, where $Z(f)= (f(s,a) - r - \gamma V_f(s'))^2 - ((\Tcal f)(s,a) - r - \gamma V_f(s'))^2$ and $\frac{1}{n}\sum_{i=1}^n Z_i(f)=\Lcal_D(f; f)-\Lcal_D(\Tcal f;f).$
	
	Secondly, by the property of $f^\star$, we can bound $\frac{1}{n}\sum_{i=1}^n Z_i(f^\star)$ by $\epsilon_3$ in Lemma \ref{lem:fast_2}.
	
	These two substeps give us the bound of $\frac{1}{n}\sum_{i=1}^nZ_i(\hat f)$.
	
	Thirdly, applying Lemma \ref{lem:iteration} and Lemma \ref{lem:decompose}, which is the similar steps in FQI, we obtain the desired result.
\end{enumerate}

We start proving Theorem~\ref{thm:minimax_full} by a concentration result. 
\begin{lemma}
\label{lem:fast_2} 
Under the same assumption as Lemma \ref{lem:fast_1}, we have that $\forall f\in\Fcal$, with probability at least $1-\delta$,
\begin{align*}
\left|\Lcal_D(\widehat\Tcal_\Gcal f;f)-\Lcal_D(\Tcal f;f)\right|\leq\frac{43\Vmax^2 \ln\tfrac{4|\Fcal|| \Gcal|}{\delta}}{n}+\sqrt{\frac{239\Vmax^2 \ln\tfrac{4|\Fcal|| \Gcal|}{\delta}}{n}\epsilon_{\Fcal,\Gcal}}+\epsilon_{\Fcal,\Gcal}.
\end{align*}
	
\end{lemma}
\begin{proof}	
	 We first apply (two-sided) Bernstein's inequality and union bound over all $f\in\Fcal$ and $g\in\Gcal$ (similar to Eq.(\ref{eq:Bernstein}) in Lemma \ref{lem:fast_1}).  Define $\delta':=\delta/4.$ With probability at least $1-2\delta'$, we have

	 
	$$
	\left|\frac{1}{n}\sum_{i=1}^n X_i(\widehat\Tcal_\Gcal,f,g_f^\star) - \EE[X(\widehat\Tcal_\Gcal, f,g_f^\star)]\right|
	\le \sqrt{\frac{16 \Vmax^2 \left(\EE[X(\widehat\Tcal_\Gcal, f,g_f^\star)]+2\epsilon_{\Fcal,\Gcal}\right) \ln\tfrac{|\Fcal|| \Gcal|}{\delta'}}{n}} + \frac{4\Vmax^2 \ln\tfrac{|\Fcal|| \Gcal|}{\delta'}}{3n},
	$$
	which means that
	$$
	\left|\frac{1}{n}\sum_{i=1}^n X_i(\widehat\Tcal_\Gcal,f,g_f^\star)\right| \le \left|\EE[X(\widehat\Tcal_\Gcal, f,g_f^\star)]\right|+
	 \sqrt{\frac{16 \Vmax^2 \left(\EE[X(\widehat\Tcal_\Gcal, f,g_f^\star)]+2\epsilon_{\Fcal,\Gcal}\right) \ln\tfrac{|\Fcal|| \Gcal|}{\delta'}}{n}} + \frac{4\Vmax^2 \ln\tfrac{|\Fcal|| \Gcal|}{\delta'}}{3n}.$$
	
	Noticing that $\EE [X(\widehat\Tcal_\Gcal f, f,g_f^\star)]=\Lcal_\muu(\widehat\Tcal_\Gcal f;f)-\Lcal_\muu(g_f^\star;f)=[\Lcal_\muu(\widehat\Tcal_\Gcal f;f)-\Lcal_\muu(\Tcal f;f)]+[\Lcal_\muu(\Tcal f; f)- \Lcal_\muu(g_f^\star;f)]=\|\widehat\Tcal_\Gcal f-\Tcal f\|_{2,\mu}^2-\|g_f^\star-\Tcal f\|_{2,\mu}^2\geq 0$, and the results in Lemma \ref{lem:fast_1} also holds (one-sided Bernstein's inequality is implied by the two-sided Bernstein's inequality), we have
	$$0\leq \Lcal_\muu(\widehat \Tcal_\Gcal f;f)-\Lcal_\muu(g_f^\star;f) \le \frac{56\Vmax^2\ln\frac{|\Fcal|| \Gcal|}{\delta'}}{3n}+\sqrt{\frac{32\Vmax^2\ln\frac{|\Fcal|| \Gcal|}{\delta'}}{n}\epsilon_{\Fcal,\Gcal}}~.$$
	Therefore, we have $$\left|\EE [X(\widehat\Tcal_\Gcal f, f,g_f^\star)]\right|\le\frac{56\Vmax^2\ln\frac{|\Fcal|| \Gcal|}{\delta'}}{3n}+\sqrt{\frac{32\Vmax^2\ln\frac{|\Fcal|| \Gcal|}{\delta'}}{n}\epsilon_{\Fcal,\Gcal}}~.$$
	Substituting this inequality into the bound of $\left|\frac{1}{n}\sum_{i=1}^n X_i(\widehat\Tcal_\Gcal,f,g_f^\star)\right|$, we have that with probability at least $1-2\delta'$,
	\begin{align*}
	&~\left|\frac{1}{n}\sum_{i=1}^n X_i(\widehat\Tcal_\Gcal,f,g_f^\star)\right|\\
	\le &~ \frac{56\Vmax^2\ln\frac{|\Fcal|| \Gcal|}{\delta'}}{3n}+\sqrt{\frac{32\Vmax^2\ln\frac{|\Fcal|| \Gcal|}{\delta'}}{n}\epsilon_{\Fcal,\Gcal}} \\
	&~+ \sqrt{\frac{16\Vmax^2 \left(\frac{56\Vmax^2\ln\frac{|\Fcal|| \Gcal|}{\delta'}}{3n}+\sqrt{\frac{32\Vmax^2\ln\frac{|\Fcal|| \Gcal|}{\delta'}}{n}\epsilon_{\Fcal,\Gcal}}+2\epsilon_{\Fcal,\Gcal}\right) \ln\tfrac{|\Fcal|| \Gcal|}{\delta'}}{n}} + \frac{4\Vmax^2 \ln\tfrac{|\Fcal|| \Gcal|}{\delta'}}{3n} \tag{*}\\
	\leq&~\frac{60\Vmax^2 \ln\tfrac{|\Fcal|| \Gcal|}{\delta'}}{3n}+\sqrt{\frac{32\Vmax^2 \ln\tfrac{|\Fcal|| \Gcal|}{\delta'}}{n}\epsilon_{\Fcal,\Gcal}}+ \sqrt{\frac{16\Vmax^2 \left(\frac{80\Vmax^2\ln\frac{|\Fcal|| \Gcal|}{\delta'}}{3n}+3\epsilon_{\Fcal,\Gcal}\right) \ln\tfrac{|\Fcal|| \Gcal|}{\delta'}}{n}} \\
	\leq&~\frac{122\Vmax^2 \ln\tfrac{|\Fcal|| \Gcal|}{\delta'}}{3n}+\sqrt{\frac{159\Vmax^2 \ln\tfrac{|\Fcal|| \Gcal|}{\delta'}}{n}\epsilon_{\Fcal,\Gcal}}~. \tag{$\sqrt{a+b}\leq\sqrt{a}+\sqrt{b}$ and $\ln\frac{|\Fcal|| \Gcal|}{\delta'}>0$}
	\end{align*}
	In Step (*), we use $\sqrt{\frac{32\Vmax^2\ln\frac{|\Fcal|| \Gcal|}{\delta'}}{n}\epsilon_{\Fcal,\Gcal}}\le\frac{8\Vmax^2\ln\frac{|\Fcal|| \Gcal|}{\delta'}}{n}+\epsilon_{\Fcal,\Gcal}$.
	
	Then, define $$ Y(g,f) := (g(s,a) - r - \gamma V_{f}(s'))^2 - ((\Tcal f)(s,a) - r - \gamma V_{f}(s'))^2.$$
	Plugging each $(s,a,r,s') \in D$ into $Y_1(g_f^\star,f)$, we get i.i.d.~variables $Y_2(g_f^\star,f), Y(g_f^\star,f), \ldots,$ $Y_n(g_f^\star,f)$.
	Applying same derivations in Lemma \ref{lem:fast_1}, we can get similar bound as Inequality (\ref{eq:bound_var}), $$0\le\VV[Y(g_f^\star,f)] \le 4\Vmax^2 \|g_f^\star-\Tcal f\|_{2,\muu}^2(=4\Vmax^2 \EE [Y(g_f^\star,f)])\leq 4\Vmax^2\epsilon_{\Fcal,\Gcal}.$$ We can apply (two-sided) Bernstein's inequality and union bound over all $f\in\Fcal$ and $g\in\Gcal$. With probability at least $1-2\delta'$, we have
	\begin{align*}
	\left|\frac{1}{n}\sum_{i=1}^n Y_i(g_f^\star,f)-\EE[Y(g_f^\star, f)]\right|
	\le &~ \sqrt{\frac{8 \Vmax^2 \EE[Y(g_f^\star, f)] \ln\tfrac{|\Fcal|| \Gcal|}{\delta'}}{n}} + \frac{4\Vmax^2 \ln\tfrac{|\Fcal|| \Gcal|}{\delta'}}{3n},
	\end{align*}
	which means that
	\begin{align*}
	\left|\frac{1}{n}\sum_{i=1}^n Y_i(g_f^\star,f)\right|\le \epsilon_{\Fcal,\Gcal}+ \sqrt{\frac{8 \Vmax^2 \ln\tfrac{|\Fcal|| \Gcal|}{\delta'}}{n}\epsilon_{\Fcal,\Gcal}} + \frac{4\Vmax^2 \ln\tfrac{|\Fcal|| \Gcal|}{\delta'}}{3n}.
	\end{align*}
	Union bounding the results of $\frac{1}{n}\sum_{i=1}^n X_i(\widehat\Tcal_\Gcal,f,g_f^\star)$ and $\frac{1}{n}\sum_{i=1}^n Y_i(g_f^\star,f)$, we have that with probability at least $1-4\delta'$,
	\begin{align*}
	&~\left|\Lcal_D(\widehat\Tcal_\Gcal f;f)-\Lcal_D(\Tcal f;f)\right|\\
	=&\left|\frac{1}{n}\sum_{i=1}^n X_i(\widehat\Tcal_\Gcal,f,g_f^\star)+\frac{1}{n}\sum_{i=1}^n Y_i(g_f^\star,f)\right|\\
	\leq&~\frac{43\Vmax^2 \ln\tfrac{|\Fcal|| \Gcal|}{\delta'}}{n}+\sqrt{\frac{239\Vmax^2 \ln\tfrac{|\Fcal|| \Gcal|}{\delta'}}{n}\epsilon_{\Fcal,\Gcal}}+\epsilon_{\Fcal,\Gcal}.
	\end{align*}
	Noticing $\delta'=\delta/4$, we complete the proof.
\end{proof}

\begin{proof}[Proof \textbf{of Theorem~\ref{thm:minimax_full}}]~~
Firstly, Lemma \ref{lem:iteration} gives us that $$\|\hat f - Q^\star\|_{2,\nu}\leq \frac{\sqrt{C}}{1-\gamma}~ \|\hat f - \Tcal \hat f\|_{2,\muu}.$$

It then suffices to upper bound $\|\hat f - \Tcal \hat f\|_{2,\muu}$.

The objective of the minimax form minimization can be written as $\inf_{f\in\Fcal}\sup_{g\in\Gcal}\left(\Lcal_D(f;f)-\Lcal_D(g;f)\right).$
We can find that, $\forall f\in\Fcal$,
$$\argmax_{g\in\Gcal}\left(\Lcal_D(f;f)-\Lcal_D(g;f)\right)=\argmax_{g\in\Gcal}-\Lcal_D(g;f)=\argmin_{g\in\Gcal}\Lcal_D(g;f)=\widehat\Tcal_{\Gcal}f.$$

Define $\delta':=\delta/2$, Lemma \ref{lem:fast_2} tells us that $\forall f\in\Fcal$, we have that with probability at least $1-\delta'$,
\begin{align*}
&~\left|\Lcal_D(\widehat\Tcal_\Gcal f;f)-\Lcal_D(\Tcal f;f)\right|\leq\epsilon_2,
\end{align*}
where
$$\epsilon_2=\frac{43\Vmax^2 \ln\tfrac{4|\Fcal|| \Gcal|}{\delta'}}{n}+\sqrt{\frac{239\Vmax^2 \ln\tfrac{4|\Fcal|| \Gcal|}{\delta'}}{n}\epsilon_{\Fcal,\Gcal}}+\epsilon_{\Fcal,\Gcal}.$$

From the approximate realizability of $\Fcal$, we know there exists $f^\star \in \Fcal,~s.t.~ \|f^\star-\Tcal f^\star\|_{2,\mu}^2\le\epsilon_\Fcal$. Then by the optimality of $\hat f$, we have that $\Lcal_D(\hat f;\hat f)-\Lcal_D(\widehat\Tcal_{\Gcal} \hat f;\hat f)\leq\Lcal_D(f^\star;f^\star)-\Lcal_D(\widehat\Tcal_{\Gcal} f^\star;f^\star)$. Therefore, with probability at least $1-\delta'$, we have that 
$$\Lcal_D(\hat f;\hat f)-\Lcal_D(\Tcal \hat f;\hat f)\leq\Lcal_D(f^\star;f^\star)-\Lcal_D(\Tcal f^\star;f^\star)+2\epsilon_2.$$

Define $$Z(f):= \left(f(s,a) - r - \gamma V_f(s')\right)^2 - \left((\Tcal f)(s,a) - r - \gamma V_f(s')\right)^2.$$
Plugging each $(s,a,r,s') \in D$ into $Z(f)$, we get i.i.d.~variables $Z_1(f), Z_2(f), \ldots,$ $Z_n(f)$.
Applying Ineq. (\ref{eq:bound_var}) in Lemma \ref{lem:fast_1}, we get
$$\VV[Z(f)] \le 4\Vmax^2 \|f-\Tcal f\|_{2,\muu}^2=4\Vmax^2 \EE [Z(f)].$$ We can apply (one-sided) Bernstein's inequality and union bound over all $f\in\Fcal$. With probability at least $1-\delta'$, we have that $\forall f\in\Fcal$,
\begin{align*}
\frac{1}{n}\sum_{i=1}^n Z_i(f)-\EE[Z(f)]
\le &~ \sqrt{\frac{8 \Vmax^2 \EE[Z(f)] \ln\tfrac{|\Fcal|}{\delta'}}{n}} + \frac{4\Vmax^2 \ln\tfrac{|\Fcal|}{\delta'}}{3n}.
\end{align*}
Substituting $f^\star$ into the inequality and noticing $\|f^\star-\Tcal f^\star\|_{2,\muu}^2\leq \epsilon_\Fcal$, we have
\begin{align*}
\frac{1}{n}\sum_{i=1}^n Z_i(f^\star)\le \epsilon_\Fcal+ \sqrt{\frac{8 \Vmax^2 \ln\tfrac{|\Fcal|}{\delta'}}{n}\epsilon_\Fcal} + \frac{4\Vmax^2 \ln\tfrac{|\Fcal|}{\delta'}}{3n}:=\epsilon_3.
\end{align*}
Since $\forall f\in\Fcal,$ $\frac{1}{n}\sum_{i=1}^n Z_i(f)=\Lcal_D(f;f)-\Lcal_D(\Tcal f;f),$ with probability at least $1-2\delta'$, we have
\begin{align*}
\frac{1}{n}\sum_{i=1}^n Z_i(\hat f)=&~\Lcal_D(\hat f;\hat f)-\Lcal_D(\Tcal \hat f;\hat f)\leq \Lcal_D(f^\star;f^\star)-\Lcal_D(\Tcal f^\star;f^\star)+2\epsilon_2
\leq2\epsilon_2+\epsilon_3.
\end{align*}

Finally, we consider $Z(\hat f)$. Our goal is to bound $\|\hat f- \Tcal \hat f\|_{2,\muu}=
\sqrt{\EE [Z(\hat f)]}$. Substituting $\hat f$ into the concentration bound of $Z(f)$, we have
\begin{align*}
	\EE[Z(\hat f)]-\frac{1}{n}\sum_{i=1}^n Z_i(\hat f)
	\le &~ \sqrt{\frac{8 \Vmax^2 \EE[Z(\hat f)] \ln\tfrac{|\Fcal| }{\delta'}}{n}} + \frac{4\Vmax^2 \ln\tfrac{|\Fcal|}{\delta'}}{3n}.
\end{align*}
Substituting the upper bound of $\frac{1}{n}\sum_{i=1}^n Z_i(\hat f)$ into the equality, we have that, with probability at least $1-2\delta'$,
\begin{align*}
	\|\hat f - \Tcal \hat f\|_{2,\muu}^2=\EE[Z(\hat f)]
	\le &~ \sqrt{\frac{8 \Vmax^2 \EE[Z(\hat f)] \ln\tfrac{|\Fcal| }{\delta'}}{n}} + \frac{4\Vmax^2 \ln\tfrac{|\Fcal|}{\delta'}}{3n}+2\epsilon_2+\epsilon_3.
\end{align*}
Solving this quadratic formula and noticing that $\delta=2\delta'$, we have that with probability at least $1-\delta$,
\begin{align*}
	\|\hat f - \Tcal \hat f\|_{2,\muu}^2
	\le &~ \frac{16 \Vmax^2  \ln\tfrac{2|\Fcal|}{\delta}}{3n} +2\epsilon_2+\epsilon_3+\sqrt{\frac{8\Vmax^2\ln\frac{2|\Fcal|}{\delta}}{n}\left(\frac{10\Vmax^2\ln\frac{2|\Fcal|}{\delta}}{3n}+2\epsilon_2+\epsilon_3\right)}~,
\end{align*}
where
$$\epsilon_2=\frac{43\Vmax^2 \ln\tfrac{8|\Fcal| |\Gcal|}{\delta}}{n}+\sqrt{\frac{239\Vmax^2 \ln\tfrac{8|\Fcal| |\Gcal|}{\delta}}{n}\epsilon_{\Fcal,\Gcal}}+\epsilon_{\Fcal,\Gcal}~,$$
and
$$\epsilon_3=\epsilon_\Fcal+ \sqrt{\frac{8 \Vmax^2 \ln\tfrac{2|\Fcal|}{\delta}}{n}\epsilon_\Fcal} + \frac{4\Vmax^2 \ln\tfrac{2|\Fcal|}{\delta}}{3n}.$$
In this way, we obtain the bound for $\|\hat f -\Tcal \hat f \|_{2,\muu}$, and further the bound for $\|f-Q^\star\|_{2,\muu}$ (by Lemma \ref{lem:iteration}). Finally, applying the bound for $\|f-Q^\star\|_{2,\muu}$  to Lemma \ref{lem:decompose}, we have that with probability at least $1-\delta$,

\begin{align*}
v^\star - v^{\pi_{\hat f}} \le&~ \frac{2\sqrt{C}}{(1-\gamma)^2}\sqrt{\frac{16 \Vmax^2  \ln\tfrac{2|\Fcal| }{\delta}}{3n} +2\epsilon_2+\epsilon_3+\sqrt{\frac{8\Vmax^2\ln\frac{2|\Fcal|}{\delta}}{n}\left(\frac{10\Vmax^2\ln\frac{2|\Fcal|}{\delta}}{3n}+2\epsilon_2+\epsilon_3\right)}}\\
\le&~\frac{2\sqrt{C}}{(1-\gamma)^2}\sqrt{\frac{16 \Vmax^2  \ln\tfrac{2|\Fcal| }{\delta}}{3n} +2\epsilon_2+\epsilon_3+\frac{2\Vmax^2\ln\frac{2|\Fcal|}{\delta}}{n}+\left(\frac{10\Vmax^2\ln\frac{2|\Fcal|}{\delta}}{3n}+2\epsilon_2+\epsilon_3\right)}\\
=&~\frac{2\sqrt{C}}{(1-\gamma)^2}\sqrt{\frac{32 \Vmax^2  \ln\tfrac{2|\Fcal| }{\delta}}{3n} +4\epsilon_2+2\epsilon_3}\\
\le&~\frac{2\sqrt{C}}{(1-\gamma)^2}\left(\sqrt{\frac{32 \Vmax^2  \ln\tfrac{2|\Fcal|}{\delta}}{3n}} +\sqrt{4\epsilon_2}+\sqrt{2\epsilon_3}\right)\\
\le&~\frac{2\sqrt{C}}{(1-\gamma)^2}\left(\sqrt{\frac{32 \Vmax^2  \ln\tfrac{2|\Fcal|}{\delta}}{3n}} +\sqrt{\frac{172\Vmax^2 \ln\tfrac{8|\Fcal||\Gcal|}{\delta'}}{n}}+\sqrt[4]{\frac{3824\Vmax^2 \ln\tfrac{8|\Fcal||\Gcal|}{\delta}}{n}\epsilon_{\Fcal,\Gcal}}+2\sqrt{\epsilon_{\Fcal,\Gcal}}\right)\\
&~+\frac{2\sqrt{C}}{(1-\gamma)^2}\left(\sqrt{2\epsilon_\Fcal}+ \sqrt[4]{\frac{32 \Vmax^2 \ln\tfrac{2|\Fcal| }{\delta}}{n}\epsilon_\Fcal} + \sqrt{\frac{8\Vmax^2 \ln\tfrac{2|\Fcal|}{\delta}}{3n}}\right)\\
\le&~\frac{2\sqrt{C}}{(1-\gamma)^2}\left(\sqrt{2\epsilon_\Fcal}+2\sqrt{\epsilon_{\Fcal,\Gcal}}\right) +
\frac{2\sqrt{C}}{(1-\gamma)^2}\left(\sqrt{\frac{24 \Vmax^2  \ln\tfrac{2|\Fcal| }{\delta}}{n}} +\sqrt{\frac{172\Vmax^2 \ln\tfrac{8|\Fcal||\Gcal|}{\delta}}{n}}\right)\\
&~+\frac{2\sqrt{C}}{(1-\gamma)^2}\left(\sqrt[4]{\frac{32 \Vmax^2 \ln\tfrac{2|\Fcal| }{\delta}}{n}\epsilon_\Fcal}+\sqrt[4]{\frac{3824\Vmax^2 \ln\tfrac{8|\Fcal||\Gcal|}{\delta}}{n}\epsilon_{\Fcal,\Gcal}}\right).
\end{align*}
The proof is completed by combining the terms and absorbing the constants using Big-Oh notation.
\end{proof}
\section{Proofs Related to State Abstractions}
\label{app:abstraction}

\subsection{Equivalence Between MBRL with State Abstractions and FQI with Piece-wise Constant Function Class}

\begin{proposition}\label{prop:mbrl_fqi}
In model-based RL with abstraction $\phi: \Scal \to \Scal_\phi$, we estimate an abstract model $\widehat M_\phi = (\Scal_\phi, \Acal, \widehat P_\phi, \widehat R_\phi, \gamma)$ and then perform planning. When value iteration is used as the planning algorithm, the procedure is exactly equivalent to FQI with $\Fcal^\phi$ as the function class.
\end{proposition}

To prove the result, we first define a few notations.

\begin{definition}[\emph{Lifting}] \label{def:lifting}
	Given the MDP $M=(\Scal,\Acal,P,R,\gamma)$ and the state abstraction $\phi$ that operates on $\Scal$,
	for any function $f$ that operates on $\Scal_\phi\times\Acal$, we use $[f]_M$ to denote its \emph{lifted} version, which is a function over $\Scal\times\Acal$ and defined as $[f]_M(s,a) := f(\phi(s),a)$.

	Similarly, we can also lift a state value function. For any function $f$ that operates on $\Scal_\phi$, we also use $[f]_M$ to denote its \emph{lifted} version, which is a function over $\Scal$ and defined as $[f]_M(s) := f(\phi(s))$. Lifting a real-valued function $f$ over states can also be expressed in vector form: $[f]_M = \Phi^\top f$, where $\Phi$ is an $|\Scal_\phi|\times |\Scal|$ matrix with entries $\phi(s, x) = \mathbb{I}[\phi(s)=x]$.
\end{definition}

\begin{definition} \label{def:piece-wise}
	For piece-wise constant state-action value function $f$ and $x\in\Scal_\phi$, define $[f]_\phi(x, a) = f(s, a)$ for any $s\in\phi^{-1}(x)$; note that the notation $[\cdot]_\phi$ can only be applied to functions that are piece-wise constant under $\phi$. 
\end{definition}

\begin{proof}[Proof of Proposition~\ref{prop:mbrl_fqi}]
Let $D=\{D_{s,a}\}_{(s,a)\in\Scal\times\Acal}$ where $D_{s,a}$ is the collection of transition tuples that start with $(s,a)$. We also let $\mathbf{e}_{\phi(s')}$ be the unit vector whose $\phi(s')$-th entry is 1 and all other entries are 0. Then, for any abstract state-action pair $(x,a) \in\Scal_\phi \times\Acal$, the certainty-equivalence estimate of model parameters are:
$$\widehat R_\phi(x,a)=\frac{1}{|D_{x,a}|}\sum_{(r,s')\in D_{x,a}}r \quad \text{and}\quad\widehat P_\phi(x,a)=\frac{1}{|D_{x,a}|}\sum_{(r,s')\in D_{x,a}} \mathbf{e}_{\phi(s')}.$$
	
	If we use value iteration as the planning algorithm, we will first initialize $g_0\in [0,\Rmax]^{|\Scal_\phi\times\Acal|}$. Then in each iteration, we let $g_t=\Tcal_{\widehat{M}_\phi}g_{t-1}$. Expanding the operator $\Tcal_{\widehat{M}_\phi}$, for $x\in\Scal_\phi$ and $a\in\Acal$, we have
	\begin{align*}
		g_t(x,a)&=\widehat R_\phi(x,a)+\gamma\langle\widehat P_\phi(x,a),V_{g_{t-1}}\rangle\\
		&=\frac{1}{|D_{x,a}|}\sum_{(r,s')\in D_{x,a}}\left(r+\gamma \langle \mathbf{e}_{\phi(s')}, V_{g_{t-1}}\rangle\right)\\
		&=\frac{1}{|D_{x,a}|}\sum_{(r,s')\in D_{x,a}}\left(r+\gamma  V_{g_{t-1}}(\phi(s'))\right)
	\end{align*}
	
	For the FQI with $\Fcal ^\phi$, we first initialize $f_0$ as any function in $\Fcal^\phi=[0,\Rmax]^{|\Scal_\phi\times\Acal|}$. Then in each iteration, we let $f_t=\widehat \Tcal_{\Fcal^\phi} f_{t-1}$. From the definition of $\widehat \Tcal_{\Fcal^\phi}$, for $s\in\Scal$ and $a\in\Acal$, we have $$f_t(s,a)=\argmin_{f\in\Fcal^\phi} \frac{1}{|D_{\phi(s),a}|}\sum_{(r,s')\in D_{\phi(s),a}}\left(f-r-\gamma V_{f_{t-1}}(s')\right)^2.$$ This is a regression problem and the solution is
	\begin{align*}
	f_t(s,a)&=\frac{1}{|D_{\phi(s),a}|}\sum_{(r,s')\in D_{\phi(s),a}}\left(r+\gamma\langle\mathbf{e}_{\phi(s')},[V_{f_{t-1}}]_\phi\rangle\right)\\
	&=\frac{1}{|D_{\phi(s),a}|}\sum_{(r,s')\in D_{\phi(s),a}}\left(r+\gamma V_{f_{t-1}}(s')\right)
	\end{align*}

	Therefore, if $f_0=[g_0]_M$, the two algorithms give us that $f_t=[g_t]_M$ for any $t$. This shows that model-based RL with abstraction $\phi$ is exactly equivalent to FQI with $\Fcal^\phi$. 
\end{proof}
	
\subsection{Proof of Equivalence Between Bisimulation and Completeness for Piece-wise Constant Function Class}
We first define approximate bisimulation, which is a generalization of Definition~\ref{def:bisimulation}.
\begin{definition}[\emph{Approximate model-irrelevant}] \label{def:approx_bisimulation}
Given the MDP $M=(\Scal,\Acal,P,R,\gamma)$ and the state abstraction $\phi: \Scal\to\Scal_\phi$, we call $\phi$ an $(\epsilon_R, \epsilon_P)$-approximate bisimulation if
\begin{align} 
& \max_{s_1, s_2: \phi(s_1) = \phi(s_2), a\in\Acal} |R(s_1, a) - R(s_2, a)| = \epsilon_R, \label{eq:eps_R} \\
& \max_{s_1, s_2: \phi(s_1) = \phi(s_2), a\in\Acal} \left\|\Phi P(s_1, a) - \Phi P(s_2, a)\right\|_1 = \epsilon_P, \label{eq:eps_P} 
\end{align}
where $\Phi$ is as defined in Definition \ref{def:lifting}.
\end{definition}

\begin{proposition}[Completeness=Bisimulation] \label{prop:bisimulation_full}
	Suppose that $\phi$ is an $(\epsilon_R,\epsilon_P)$-approximate $Q^\star$-irrelevant abstraction, then we have
	$$\max\left\{\frac{\epsilon_R}{2},\frac{\gamma\epsilon_P \Vmax}{4}\right\}\leq\sup_{f\in\Fcal^\phi}\inf_{f'\in\Fcal^\phi}\|f'-\Tcal f\|_\infty\leq\frac{\epsilon_R}{2}+\frac{\gamma\epsilon_P \Vmax}{4}.$$
\end{proposition}

Proposition~\ref{prop:bisimulation} is a direct corollary of the above result when $\epsilon_R, \epsilon_P, \sup_{f\in\Fcal^\phi}\inf_{f'\in \Fcal^\phi}\|f' - \Tcal f\|_\infty$ are all $0$'s. In the approximate case, however, we use $\epsilon_R$ and $\epsilon_P$ to provide both upper and lower bounds of the violation of completeness, but do not obtain an equality relationship. This is purely an artifact that bisimulation considers rewards and transitions separately, whereas completeness considers both of them together in terms of the Bellman update operator $\Tcal$, and cancellation between reward/transition errors may occur.



\begin{proof}[Proof of Proposition~\ref{prop:bisimulation_full}]

We first prove the upper bound. For any fixed $f\in\Fcal^\phi$, we show that there exists $f_1' \in \Fcal^\phi$ such that $\|f_1'-\Tcal f\|_\infty \le \epsilon_R/2+\gamma\epsilon_P\Vmax /4$. Therefore $\inf_{f'\in\Fcal^\phi}\|f'-\Tcal f\|_\infty \le \|f_1'-\Tcal f\|_\infty$ and hence is subject to the same upper bound.
 	
Since $f_1'$ is required to be piece-wise constant, it suffices to specify $[f_1']_\phi(x,a)$ for each $x\in\Scal_\phi,a\in\Acal$. Fixing any $x, a$, define 
\begin{equation} \label{eq:s_min_s_max}
\smax :=\argmax_{s\in\phi^{-1}(x), \, a\in \Acal} (\Tcal f)(s, a), \qquad \smin :=\argmin_{s\in\phi^{-1}(x), \, a\in \Acal} (\Tcal f)(s, a),
\end{equation}
and 
\begin{align} \label{eq:f_1'}
[f_1']_\phi(x,a) := \tfrac{1}{2}\left((\Tcal f)(\smax, a) + (\Tcal f)(\smin, a)\right).
\end{align}
Note that $f_1'\in[0,\Vmax]$ so $f_1' \in \Fcal^\phi$. 
It remains to upper bound $\|f_1' -\Tcal f\|_\infty$. 

For any $s\in\Scal, a\in\Acal$, let $x=\phi(s)$, and $\smax$ and $\smin$ as defined in Eq.\eqref{eq:s_min_s_max} for $(x,a)$,
\begin{align*}
&~ f_1'(s, a) - (\Tcal f)(s, a) \\
\le &~ \tfrac{1}{2}\left((\Tcal f)(\smax, a) + (\Tcal f)(\smin, a)\right) - (\Tcal f)(\smin, a)  \tag{Eq.\eqref{eq:s_min_s_max} and \eqref{eq:f_1'}}\\
= &~ \tfrac{1}{2}\left((\Tcal f)(\smax, a) - (\Tcal f)(\smin, a)\right) \\
= &~ \tfrac{1}{2} \left(R(\smax,a) + \gamma\langle P(\smax,a),V_f\rangle - R(\smin,a) - \gamma\langle P(\smin,a),V_f\rangle \right)\\
\le &~ \tfrac{1}{2} \left|R(\smax,a)-R(\smin,a)\right|+ \tfrac{\gamma}{2} \left|\langle P(\smax,a)-P(\smin,a),V_f\rangle\right|\\
\le &~ \tfrac{1}{2} \epsilon_R +\tfrac{\gamma}{2} \left|\langle \Phi P(\smax,a)- \Phi P(\smin,a),[V_f]_\phi\rangle\right| \tag{$f$ is piece-wise constant and so is $V_f$} \\
=&~ \tfrac{1}{2} \epsilon_R +\tfrac{\gamma}{2}\left|\left\langle \Phi P(\smax,a) - \Phi P(\smin,a),[V_f]_\phi - \tfrac{\Vmax}{2} \cdot \mathbf{1} \right\rangle\right| \tag{*}\\
\le &~ \tfrac{1}{2} \epsilon_R +\tfrac{\gamma}{2} \|\Phi P(\smax,a) - \Phi P(\smin,a)\|_1 \cdot \Vmax/2 \tag{H\"older's inequality}\\
\le &~ \frac{\epsilon_R}{2}  +\frac{\gamma \epsilon_P \Vmax}{4}. 
\end{align*}
Here Step (*) holds because $\langle \Phi P(\smax,a)- \Phi P(\smin,a),\mathbf{1} \rangle = 0$, as $\Phi P(\cdot, \cdot)$ is always a valid distribution. The other direction follows exactly the same argument due to symmetry and is omitted, and together we conclude that $\|f_1' -\Tcal f\|_\infty \leq\frac{\epsilon_R}{2}+\frac{\gamma\epsilon_P \Vmax}{4}$. 

We then turn to the lower bound of the theorem statement. It suffices to show $\exists f_R, f_P \in \Fcal^\phi$, such that $\inf_{f'\in\Fcal^\phi}\|f' - \Tcal f_R\|_\infty \ge \epsilon_R/2$ and $\inf_{f'\in\Fcal^\phi}\|f' - \Tcal f_P\|_\infty \ge\gamma  \epsilon_P \Vmax/4$, respectively. 

\paragraph{Case of $f_R$} 
Let $f_R := \mathbf{0} \in \Fcal^\phi$, so $
\inf_{f'\in\Fcal^\phi}\|f' - \Tcal f_R\|_\infty = \inf_{f'\in\Fcal^\phi}\|f' - R\|_\infty$, where $R \in [0, \Rmax]^{|\Scal\times\Acal|}$ is the reward function. It is obvious from the definition of $\epsilon_R$ in Eq.\eqref{eq:eps_R} that $\inf_{f'\in\Fcal^\phi}\|f' - R\|_\infty = \epsilon_R/2$, which proves the result.
		
\paragraph{Case of $f_P$} 
Let $s_1^P,s_2^P\in\Scal, a^P\in\Acal$ be the arguments that achieve the maximum in Eq.\eqref{eq:eps_P}, i.e., $\phi(s_1^P) = \phi(s_2^P)$ and $\|\Phi P(s_1^P, a^P) -  \Phi P(s_2^P, a^P)\|_1 = \epsilon_P$. 
We construct $f_P \in\Fcal^\phi$ as follows: Assume w.l.o.g.~that $R(s_1^P, a^P) \ge R(s_2^P, a^P)$. For any $x\in\Scal_\phi$ and $a\in\Acal$, define
$$
[f_P]_\phi(x,a)= \Indi[P(x|s_1^P,a^P)>P(x|s_2^P,a^P)] \cdot \Vmax.
$$

Note that the RHS has no dependence on $a$, so $V_f(s) = [f_P]_\phi(\phi(s), a)$ for any $a\in\Acal$. It is easy to verify that $f_P \in \Fcal^\phi$ as its value is either $0$ or $\Vmax$. Essentially $f_P$ is designed such that $[V_{f_P}]_\phi$ witnesses the $\ell_1$ error (or total variation) between $\Phi P(s_1^P, a^P)$ and $\Phi P(s_2^P, a^P)$. The inequality sign inside the indicator could be either ``$>$'' or ``$<$'', and we choose it in consistence with the relationship between $R(s_1^P, a^P)$ and  $R(s_2^P, a^P)$, which guarantees that the reward error and the transition error wouldn't cancel out with each other. Now consider the difference between two entries in $(\Tcal f_P)$: 
\begin{align*}
&~(\Tcal f_P)(s_1^P,a^P)-(\Tcal f_P)(s_2^P,a^P)\\
= &~ R(s_1^P,a^P) + \gamma\langle P(s_1^P,a^P),V_{f_P}\rangle -R(s_2^P,a^P) - \gamma\langle P(s_2^P,a^P),V_{f_P}\rangle \\
= &~ (R(s_1^P,a^P)-R(s_2^P,a^P)) + \gamma\langle P(s_1^P,a^P)-P(s_2^P,a^P),V_{f_P} \rangle  \tag{Both terms are positive due to construction}\\
= &~ \left|R(s_1^P,a^P)-R(s_2^P,a^P)\right|+\gamma\left|\langle \Phi P(s_1^P,a^P)-\Phi P(s_2^P,a^P),[V_{f_P}]_\phi\rangle\right| \\
\ge &~ 0 + \gamma\|\Phi P(s_1^P,a^P)-\Phi P(s_2^P,a^P)\|_{TV}\Vmax \tag{*} \\
= &~ \gamma \epsilon_P\Vmax/2.
\end{align*}

Step (*) follows because $[V_{f_P}]_\phi$ takes $\Vmax$ on the subset of $\Scal_\phi$ where $\Phi P(s_1^P, a^P)$ has a greater probability than $\Phi P(s_2^P, a^P)$, and $0$ otherwise, so the dot product is equal to the total variation up to the scaling factor of $\Vmax$. 

Now $\sup_{f\in\Fcal^\phi}\inf_{f'\in\Fcal^\phi}\|f'-\Tcal f\|_\infty\geq \inf_{f'\in\Fcal^\phi}\|f'-\Tcal f_P\|_\infty$, which is the approximation error of $\Tcal f_P$ in $\Fcal^\phi$. Since $\Tcal f_P$ takes values that are  $\gamma \epsilon_P\Vmax/2$ apart for aggregated states $s_1^P$ and $s_2^P$ on action $a^P$, the approximation error is at least $\gamma \epsilon_P\Vmax/4$. This completes the proof.
\end{proof}

\section{Proof of Proposition~\ref{prop:value_profile}}
\label{app:value_profile}
Firstly we introduce a standard result that bounds the loss of acting greedily with respect to an approximate Q-value function.
\begin{lemma}\label{lem:qtov}\cite{singh1994upper} For any  $f:\Scal\times\Acal\rightarrow \RR$, let $\pi_f$ be its greedy policy, then
	$$\left\|V^{\star}-V^{\pi_{f}}\right\|_{\infty}\leq \frac{2\left\|f-Q^{\star}\right\|_{\infty}}{1-\gamma}.$$
\end{lemma}

Now we are ready to prove the proposition.
\begin{proof}[Proof of Proposition \ref{prop:value_profile}]
We represent each state $s$ by its value profile $\{f(s,a): f\in\Fcal, a\in\Acal\}$, estimate a tabular model, and output its optimal policy. Since a set of states may share exactly the same value profile, we are essentially using a state abstraction, denoted as $\phi$. For any two states $s_1,s_2\in\Scal$ that share the same value profile, the realizability assumption implies that $Q^\star(s_1,a)=Q^\star(s_2,a), \forall a\in\Acal$, so $\phi$ is \emph{$Q^\star$-irrelevant} \citep{li2006towards}. In the following, we will show that certainty-equivalence with \emph{$Q^\star$-irrelevant} abstraction is consistent and enjoys polynomial sample complexity if each state-action pair $(s,a)$ receives $\Omega(|D|/|\Scal\times\Acal|)$ data.
	
	Let $D_{s,a}$ be the collection of transition tuples that start with $(s,a)$ and $D_{x,a} := \sum_{s \in \phi^{-1}(x)} |D_{s,a}|$. We first consider an abstract MDP $M_\phi = (\Scal_\phi, \Acal, P_\phi, R_\phi, \gamma)$, where $\Scal_\phi$ is the abstract state space (isomorphic to the set of distinct value profiles),
	$$ 
	R_\phi(x,a) = \frac{\sum_{ s \in \phi^{-1}(x)} |D_{ s,a}| R(s,a) }{|D_{\phi(s), a}|}, \quad
	P_\phi(x'|x,a) = \frac{\sum_{ s \in \phi^{-1}(x)} |D_{ s,a}| P(x'|s,a) }{|D_{\phi(s), a}|},\quad \forall x,x'\in\Scal_\phi, a\in\Acal.
	$$
	
	Recall the notations in Definitions~\ref{def:lifting}, \ref{def:piece-wise}, and \ref{def:approx_bisimulation}. We claim that $[Q_{M_\phi}^\star]_M = Q_M^\star$, where $[Q_{M_\phi}^\star]_M$ is the lifted version of $Q_{M_\phi}^\star$:  Since $Q_M^\star(s,a)$ is piece-wise constant under $\phi$, we let $[Q_M^\star]_\phi(x,a) = Q_M^\star(s,a)$ for any $s\in\phi^{-1}(x)$. It suffices to show $Q_{M_\phi}^\star=[Q^\star_M]_\phi$, by showing that $[Q_M^\star]_\phi$ is the fixed point of $\Tcal_{M_\phi}$. This is because, for any $x\in\Scal_\phi, a\in\Acal$, 
	\begin{align*}
	(\Tcal_{M_\phi} [Q_M^\star]_\phi)(x, a) 
	= &~ R_{\phi}(x,a) + \gamma \langle P_{\phi}(x,a), [V_M^\star]_\phi \rangle \\
	= &~ \sum_{ s \in \phi^{-1}(x)} \frac{|D_{s,a}|}{|D_{\phi(s),a}|} \left( R( s, a) + \gamma \langle \Phi P( s, a), [V_M^\star]_\phi \right))\\
	= &~ \sum_{ s \in \phi^{-1}(x)} \frac{|D_{s,a}|}{|D_{\phi(s),a}|} \left( R( s, a) + \gamma \langle P( s, a), V_M^\star \right)) \\
	= &~ \sum_{ s \in \phi^{-1}(x)} \frac{|D_{s,a}|}{|D_{\phi(s),a}|}\, [Q_M^\star]_\phi(x,a) =  [Q_M^\star]_\phi(x,a).
	\end{align*}
	
	Then we consider the estimated model using the abstract representation $\emp M_\phi = (\Scal_\phi, \Acal, \emp P_\phi, \emp R_\phi, \gamma)$. Let $\mathbf{e}_{\phi(s')}$ be the unit vector whose $\phi(s')$-th entry is 1 and all other entries are 0, the parameters are
	$$\widehat R_\phi(x,a)=\frac{1}{|D_{x,a}|}\sum_{(r,s')\in D_{x,a}}r \quad \text{and}\quad\widehat P_\phi(x,a)=\frac{1}{|D_{x,a}|}\sum_{(r,s')\in D_{x,a}} \mathbf{e}_{\phi(s')},\quad \forall (x,a) \in\Scal_\phi \times\Acal.$$
	Define the minimal samples received by any abstract state-action pair as $
	n_\phi(D) := \min_{x \in \Scal_\phi, a\in \Acal} |D_{x,a}|.$ We can upper bound $\left\|[Q_{M_\phi}^\star]_M - [Q_{\emp M_\phi}^\star]_M\right\|_\infty$ by a function of $n_\phi(D)$.
	
	Noticing the contraction property of $ \Tcal_{\emp M_\phi}$, we have
	\begin{align*}
	\left\|[Q_{M_\phi}^\star]_M - [Q_{\emp M_\phi}^\star]_M\right\|_\infty 
	=&~ \left\|Q_{M_\phi}^\star - Q_{\emp M_\phi}^\star\right\|_\infty \\
	\le &~ \frac{1}{1-\gamma} \maxnorm{Q_{M_\phi}^\star - \Tcal_{\emp M_\phi} Q_{M_\phi}^\star}\tag{*}\\
	=&~ \frac{1}{1-\gamma} \maxnorm{\Tcal_{\emp M_\phi} Q_{M_\phi}^\star - \Tcal_{M_\phi} Q_{M_\phi}^\star}.
	\end{align*}
	Step(*) holds because $$\left\|Q_{M_\phi}^\star - Q_{\emp M_\phi}^\star\right\|_\infty\le \left\|Q_{M_\phi}^\star-\Tcal_{\emp M_\phi} Q_{M_\phi}^\star\right\|_\infty+\left\|\Tcal_{\emp M_\phi} Q_{M_\phi}^\star- \Tcal_{\emp M_\phi}Q_{\emp M_\phi}^\star  \right\|_\infty\le \left\|Q_{M_\phi}^\star-\Tcal_{\emp M_\phi} Q_{M_\phi}^\star\right\|_\infty+\gamma\left\| Q_{M_\phi}^\star- Q_{\emp M_\phi}^\star  \right\|_\infty.$$
	
	Then we plug in the definition of $\Tcal_{\emp M_\phi}$ and $\Tcal_{ M_\phi}$. For each $(x,a)\in\Scal_\phi \times \Acal$, 
	\begin{align*}
	&~ |(\Tcal_{\emp M_\phi} Q_{M_\phi}^\star)(x,a) - (\Tcal_{M_\phi} Q_{M_\phi}^\star)(x,a)|  \\
	= &~ |\emp R_\phi(x,a) + \gamma \langle \emp P_\phi(x,a), V_{M_\phi}^\star \rangle - R_\phi(x,a) - \gamma \langle P_\phi(x,a), V_{M_\phi}^\star \rangle | \\
	= &~ \left|\frac{1}{|D_{x,a}|}\sum_{s \in \phi^{-1}(x)}~ \sum_{(r,s') \in D_{s,a}} \left(r + \gamma V_{M_\phi}^\star(\phi(s')) - R(s,a) - \gamma \langle P(s,a), [V_{M_\phi}^\star]_M \rangle \right) \right|.
	\end{align*}
	If we view the nested sum as a flat sum, the expression is the sum of the differences between random variables $r + \gamma V_{M_\phi}^\star(s')$ and their expectation w.r.t.~the randomness of $(r,s')$. Each sample is independent and bounded in $[0,\Vmax]$, so Hoeffding's inequality applies: with probability at least $1-\delta/|\Scal_\phi\times \Acal|$, 
	$$
	\left|(\Tcal_{\emp M_\phi} Q_{M_\phi}^\star)(x,a) - (\Tcal_{M_\phi} Q_{M_\phi}^\star)(x,a)\right| \le \Vmax\sqrt{\frac{1}{2n_\phi(D)}\ln\frac{2|\Scal_\phi\times \Acal|}{\delta}}.
	$$
	Union bounding over all $(x,a)\in\Scal_\phi\times\Acal$, with probability at least $1-\delta$, we get
	$$
	\maxnorm{\Tcal_{\emp M_\phi} Q_{M_\phi}^\star - \Tcal_{M_\phi} Q_{M_\phi}^\star} \le \Vmax\sqrt{\frac{1}{2n_\phi(D)}\ln\frac{2|\Scal_\phi\times \Acal|}{\delta}}.
	$$
	
	Therefore, with probability at least $1-\delta$, we have
	\begin{align*}
	\left\|Q_M^\star - [Q_{\emp M_\phi}^\star]_M\right\|_\infty
	&~\le \left\|Q_M^\star - [Q_{M_\phi}^\star]_M \right\|_\infty + \left\|[Q_{M_\phi}^\star]_M - [Q_{\emp M_\phi}^\star]_M\right\|_\infty\\
	&~\le \frac{\Vmax}{1-\gamma}\sqrt{\frac{1}{2n_\phi(D)}\ln\frac{2|\Scal_\phi\times \Acal|}{\delta}}.
	\end{align*}
	
	Finally, applying Lemma \ref{lem:qtov} with $f=[Q_{\emp M_\phi}^\star]_M$, we get
	$$v_M^\star-v_M^{[\pi_{\emp M_\phi}^\star]_M}\le\left\|V^\star_M-V_M^{[\pi_{\emp M_\phi}^\star]_M}\right\|_\infty=\left\|V^\star_M-V_M^{\pi_{[Q_{\emp M_\phi}^\star]_M}}\right\|_\infty\le \frac{2\left\|[Q_{\emp M_\phi}^\star]_M-Q^\star_M\right\|_\infty}{1-\gamma}.$$
	
	This means that with probability at least  $1-\delta$, the output of certainty-equivalence with \emph{$Q^\star$-irrelevant} abstraction $[\pi_{\emp M_\phi}^\star]_M$ satisfies
	 $$v_M^\star-v_M^{[\pi_{\emp M_\phi}^\star]_M}\le \frac{2\Vmax}{(1-\gamma)^2}\sqrt{\frac{1}{2n_\phi(D)}\ln\frac{2|\Scal_\phi\times \Acal|}{\delta}}\le\frac{2\Vmax}{(1-\gamma)^2}\sqrt{\frac{1}{2n_\phi(D)}\ln\frac{2|\Scal\times \Acal|}{\delta}}.$$
	 
 	We complete the proof by noticing that $n_\phi(D)= \Omega(|D|/|\Scal\times\Acal|)$, so to guarantee the above bound to be $\epsilon$, the necessary sample size $|D|$ will be polynomial in all relevant parameters.
\end{proof}

\section{Possible Relaxation of Assumption~\ref{asm:concentratability}}
\label{app:relax}

We illustrate the possibility of relaxing Assumption~\ref{asm:concentratability} using a simple example on state abstractions: when learning with abstractions, it is sufficient to have data that is relatively uniform over the abstract state space, even if some raw state receives no data. Due to the connection to FQI (Section~\ref{sec:abstraction}), one would expect that with $\Fcal^\phi$ as the function class,  concentratability coefficient can be upper bounded by the number of abstract states and incur no dependence on the raw state space, which is unfortunately not the case according to the current definition. It turns out that our analysis provides an easy fix to this issue: the proof of Theorem \ref{thm:fqi} (for FQI)  only depends on Assumption~\ref{asm:concentratability} via 
\begin{align} \label{eq:con_relax}
\|f - \Tcal f'\|_{2,\nu}\le \sqrt{C} \|f - \Tcal f'\|_{2,\muu}, \forall f, f' \in\Fcal.
\end{align}
And the proof of Theorem \ref{thm:minimax} (for the minimax algorithm) only depends on Assumption~\ref{asm:concentratability} via 
\begin{align} \label{eq:con_relax2}
\|f - \Tcal f\|_{2,\nu}\le \sqrt{C} \|f - \Tcal f\|_{2,\muu}, \forall f \in\Fcal.
\end{align}

If we define $C$ through the above inequalities (which are strict relaxations of Assumption~\ref{asm:concentratability}), Theorems~\ref{thm:fqi} and \ref{thm:minimax} still hold under Eq.\eqref{eq:con_relax} and \eqref{eq:con_relax2} respectively. Furthermore, when $\Fcal=\Fcal^\phi$ and $\phi$ is a bisimulation, we can easily verify that $C$ can be upper bounded by the number of abstract state-action pairs with uniform data. One issue here is that Eq.\eqref{eq:con_relax} is specialized to the completeness assumption, and if we wish to work with alternative assumptions (as discussed earlier), we may need to relax the definition in a different manner. Another interesting observation is that Eq.\eqref{eq:con_relax2} is less strict and much nicer than Eq.\eqref{eq:con_relax}, but so far we have not been able to modify the FQI analysis to work under Eq.\eqref{eq:con_relax2}. It is unclear whether this is an artifact of proof techniques or a fundamental difference between FQI and the minimax algorithm. We leave the investigation of these issues to future work.

\end{document}